\RequirePackage{etoolbox}
\csdef{input@path}%
{%
 {sty/},
}
\documentclass{elsevierbook}

\usepackage{natbib}
\usepackage{sty_buskulic}
\usepackage[algo2e,ruled,vlined,linesnumbered]{algorithm2e}

\begin{document}


\Mainmatter







\begin{frontmatter}
\chapter{Implicit Regularization of the Deep Inverse Prior Trained with Inertia}%
\begin{aug}
\author[addressrefs={ad1}]%
{%
\fnm{Nathan} \snm{Buskulic}%
\footnote{Corresponding author. E-mail: {nathan.buskulic@proton.me}}%
}%
\author[addressrefs={ad1}]%
{%
\fnm{Jalal} \snm{Fadili}%
\footnote{Contributing author. E-mail: Jalal.Fadili@ensicaen.fr}
}%
\author[addressrefs={ad1}]%
{%
\fnm{Yvain} \snm{Qu{\'e}au}%
\footnote{Contributing author. E-mail: yvain.queau@ensicaen.fr}
}%
\address[id=ad1]%
{%
Greyc, Normandie Univ., UNICAEN, ENSICAEN, CNRS, 6 Boulevard Maréchal Juin, Caen, 14000 France
}%
\end{aug}

\begin{abstract}
Solving inverse problems with neural networks benefits from very few theoretical guarantees when it comes to the recovery guarantees. We provide in this work convergence and recovery guarantees for self-supervised neural networks applied to inverse problems, such as Deep Image/Inverse Prior, and trained with inertia featuring both viscous and geometric Hessian-driven dampings. We study both the continuous-time case, i.e., the trajectory of a dynamical system, and the discrete case leading to an inertial algorithm with an adaptive step-size. We show in the continuous-time case that the network can be trained with an optimal accelerated exponential convergence rate compared to the rate obtained with gradient flow. We also show that training a network with our inertial algorithm enjoys similar recovery guarantees though with a less sharp linear convergence rate.
\end{abstract}

\begin{keywords}
\kwd{Deep Inverse Prior}
\kwd{Implicit regularization}
\kwd{Self-supervised}
\kwd{Inverse problems}
\kwd{Momentum}
\kwd{Hessian damping}
\kwd{Convergence}
\kwd{Stable recovery}
\end{keywords}
\end{frontmatter}

\section{Introduction}\label{sec:intro}

\subsection{Motivation}

An ubiquitous problem in science and engineering is to retrieve an unknown signal $\xv\in\R^n$ from a noisy indirect observation $\yv\in\R^m$. This inverse problem in the linear, finite-dimensional setting is formalized with a forward operator $\fop: \R^n \to \R^m$ and some additive noise $\veps$ as solving the following equation:
\begin{align}\label{eq:prob_inv}
    \yv = \fop\xvc + \veps.
\end{align}
Throughout this paper, and without loss of generality, we will assume that $\yv \in \ran{\fop}$.

While the variational model-based approach with hand-crafted regularizers has been the dominated approach for years to solve \eqref{eq:prob_inv}, data-driven approaches have emerged as powerful methods to solve inverse problems by capturing the prior information directly from data, either partly or completely, explicitly or implicitly. This trend has witnessed a dramatic increase with the rise of machine learning and notably (deep) neural networks~\cite{arridge_solving_2019,ongie_deep_2020}. This type of approach has been applied to a variety of problems, and more specifically to solve imaging problems. These networks are simply parametrized functions where the parameters are learned through some gradient-based optimization algorithm to minimize a loss function that depends on the task at hand. Many works have been devoted to the practical aspects of neural networks for inverse problems (see our review later), from the best architecture for a given task to the evaluation of such models. However, while they now yield impressive results for various problems, the theoretical understanding of their recovery properties remains largely lacking.

Our focus in this chapter is the Deep Inverse/Image Prior (DIP), that was introduced in~\cite{ulyanov_deep_2020} for simple image processing tasks (denoising, super-resolution and in-painting). The central idea in the DIP is to train a neural network which acts as a generator with a randomly generated input that can be thought of as a latent random variable in dimension much smaller than $n$. The hope is that the architecture of this neural network will induce some ``implicit regularization'' and will add more and more detailed content during training before overfitting to noise. This already highlights the necessity of an early-stopping strategy that we will make rigorous later in this chapter. The DIP approach has some advantages as it is self-supervised, and does account for the forward model, hence ensuring consistency with observations. Furthermore, it is easy to implement with very good empirical results if an appropriate network architecture is chosen for the task at hand. Recently, we provided convergence and recovery guarantees of the DIP with general loss functions when the network's parameters are trained through gradient flow~\cite{buskulic2024convergenceJournal} or gradient descent~\cite{buskulic2024descentarxiv}. In practice however, the parameters are trained through inertia-based methods (such as the widely used ADAM~\cite{kingma2014adam}) as they provide empirically faster convergence rates. Inertia-based methods have been actively studied and are known to provably lead to accelerated rates in the convex and strongly convex cases. Motivated by this, we propose to study the trajectories of the DIP neural network parameters when they are trained using inertial optimization dynamics, both in the continuous-time and discrete settings.

\subsection{Problem statement}
We will consider a feed-forward network $\gv: (\uv,\thetav) \in \R^d\times \R^p \mapsto \xv \in \R^n$, equipped with some nonlinear activation function $\phi$, that transforms an input $\uv\in\R^d$ into a vector $\xv\in\R^n$. We will restrict ourselves to fully connected multilayer networks that are defined as follows:

\begin{definition}\label{def:nn}
Let $d,L\in \N$ and $\phi : \R \to \R$ an activation map which acts componentwise on the entries of a vector. A fully connected multilayer neural network with input dimension $d$, $L$ layers and activation $\phi$, is a collection of weight matrices $\pa{\Wv^{(l)}}_{l \in [L]}$ and bias vectors $\pa{\bv^{(l)}}_{l \in [L]}$, where $\Wv^{(l)} \in \R^{N_l\times N_{l-1}}$ and $\bv^{(l)} \in \R^{N_l}$, with $N_0=d$, and $N_l \in \N$ is the number of neurons for layer $l \in [L]$. Let us gather these parameters as
	\[
	\thetav=\pa{(\Wv^{(1)},\bv^{(1)}), \ldots, (\Wv^{(L)},\bv^{(L)})} \in \bigtimes_{l=1}^L \pa{\pa{\R^{N_l \times N_{l-1}}} \times \R^{N_l}}.
	\]
Then, a neural network parametrized by $\thetav$ produces a function   
\begin{align*}
\gv: & ~ (\uv,\thetav) \in \R^d \times \bigtimes_{l=1}^L \pa{\pa{\R^{N_l \times N_{l-1}}} \times \R^{N_l}} \mapsto \gv(\uv,\thetav) \in \R^{N_L} , \qwithq N_L = n ,
	\end{align*}
	which can be defined recursively as 
	\begin{align*}
		\begin{cases}
			\gv^{(0)}(\uv,\thetav)&= \uv,
			\\
			\gv^{(l)}(\uv,\thetav)&=\phi\pa{\Wv^{(l)}\gv^{(l-1)}(\uv,\thetav)+\bv^{(l)}}, \quad \text{ for } l=1,\ldots , L-1,
			\\
			\gv(\uv,\thetav)&= \Wv^{(L)} \gv^{(L-1)}(\uv,\thetav) + \bv^{(L)}.
		\end{cases}
	\end{align*}
\end{definition}

The parameters $\thetav$ of the network are a solution of
\begin{equation}\label{eq:minP}
\min_{\theta \in \R^p} \lossy(\fop\gv(\uv,\thetav))
\end{equation}
where the loss function $\lossy:\funspacedef{\R^m}{\R_+}, \fop\gv(\uv,\thetav) \mapsto \lossy(\fop\gv(\uv,\thetav))$ measures the discrepancy between the observation $\yv$ and the observed solution of the network $\fop\gv(\uv,\thetav)$. In this work, we will use the Mean Square Error (MSE) as the loss function (see \ref{ass:l_MSE}).

We will first study the behavior of the network parameters trajectory in time when trained using the second-order ODE
\begin{equation}\tag{\tcb{$\mathrm{DIN}$}}\label{eq:DIN}
\begin{cases}
\thetavddott + \alpha \thetavdott + \beta\derivgradlosstheta + \gradlosstheta = 0 \\
\thetav(0) = \thetav_0, \dot{\thetav}(0) = \mathbf{0},
\end{cases} 
\end{equation}
where $\alpha,\beta\geq0$. This system is coined \textit{Dynamical Inertial Newton-like} (DIN) after~\cite{alvarez2002second}. The parameter $\alpha$ corresponds to viscous damping while $\beta$ is that of  geometric Hessian-driven damping. When $\beta=0$, one recovers the celebrated Polyak Heavy-Ball (HBF) method with friction~\cite{polyak_methods_1964}. Taking $\beta > 0$ has been shown to attenuate the transversal oscillations that HBF can suffer from. The system \eqref{eq:DIN} is known to achieve optimal accelerated convergence rates in both the convex and strongly convex cases when compared to gradient flow~\cite{attouch2022first}.

\subsection{General notations}

For a matrix $\Mv \in \R^{a \times b}$ we denote by $\sigmin(\Mv)$ and $\sigmax(\Mv)$ its smallest and largest non-zero singular values, and by $\kappa(\Mv) = \frac{\sigmax(\Mv)}{\sigmin(\Mv)}$ its condition number. We abuse this notation for the forward operator $\fop$ and note its minimum singular value as $\sigminA$. We also denote by $\dotprod{}{}$ the Euclidean scalar product, $\norm{\cdot}$ the associated norm (the dimension is implicit from the context), and $\normf{\cdot}$ the Frobenius norm of a matrix. With a slight abuse of notation $\norm{\cdot}$ will also denote the spectral norm of a matrix. We use $\Mv^i$ (resp. $\Mv_i$) as the $i$-th row (resp. column) of $\Mv$. We denote the Kronecker product of matrices as $\otimes$. For two vectors $\xv,\zv$, $[\xv,\zv]=\enscond{(1-\rho)\xv+\rho\zv}{\rho \in [0,1]}$ is the closed segment joining them. We use the notation $a \gtrsim b$ (resp. $a \lesssim b$) if there exists a constant $C > 0$ such that $a \geq C b$ (resp. $a \leq C b$).

We also define $\yvt = \fop\gdipt$, $\xv(t) = \gdipt$ and $\yvc = \fop(\xvc)$. The Jacobian of the network is denoted $\Jg$. The local Lipschitz constant of a mapping on a ball of radius $R > 0$ around a point $\zv$ is denoted $\Lip_{\Ball(\zv,R)}(\cdot)$. We omit $R$ in the notation when the Lipschitz constant is global.

For some $\Theta \subset \R^p$, we define $\Sigma_\Theta = \enscond{\gv(\uv,\thetav)}{\thetav \in \Theta}$ as the set of signals that the network $\gv$ can generate for all $\theta$ in the parameter set $\Theta$. $\Sigma_\Theta$ can thus be viewed as a parametric manifold. If $\Theta$ is closed (resp. compact), so is $\Sigma_\Theta$. We denote $\dist(\cdot,\Sigma_\Theta)$ the distance to $\Sigma_\Theta$ which is well defined if $\Theta$ is closed and non-empty. For a vector $\xv$, $\xvsigmatheta$ is its projection on $\Sigma_\Theta$, i.e. $\xvsigmatheta \in \Argmin_{\zv \in \Sigma_\Theta} \norm{\xv-\zv}$. Observe that $\xvsigmatheta$ always exists but might not be unique. We also define $T_{\Sigma_\Theta}(\xv)$ the tangent cone of $\Sigma_\Theta$ at $\xv\in\Sigma_\Theta$. The minimal (conic) singular value of a matrix $\fop \in \R^{m \times n}$ w.r.t. the cone $T_{\Sigma_\Theta}(\xv)$ is then defined as
\[
\lmin(\fop;T_{\Sigma_\Theta}(\xv)) = \inf
\{\norm{\fop \zv}/\norm{\zv}:  \zv \in T_{\Sigma_\Theta}(\xv)\}.
\]

\subsection{Contributions}
We provide a theoretical analysis of the recovery properties of the DIP model for solving linear inverse problems when trained using the inertial system \eqref{eq:DIN} in continuous-time, or the corresponding discretized algorithm. In the continuous-time setting, we show that the network can be trained to zero-loss with an (optimal) accelerated exponential convergence rate compared the gradient flow case, as seen in practice, at the cost of a slightly stronger condition on the initialization. We also give an early-stopping bound to avoid overfitting and an accelerated recovery result in the signal space. We show how a sufficiently overparametrized two layer Deep Inverse Prior (DIP) network~\cite{ulyanov_deep_2020} can meet the conditions to benefit from these guarantees. We also provide an inertial algorithm obtained by appropriate discretization of the continuous-time system. When the algorithm is run with an adaptive step-size to compensate for the lack of global Lipschitz smoothness, we demonstrate that the network can be trained while maintaining comparable recovery guarantees. However, unlike the continuous-time setting, the convergence rate we obtain in the algorithmic case, though linear, is not the optimal accelerated one.


\section{Prior Work}\label{sec:prior}

\noindent{\textbf{Data-Driven Methods to Solve Inverse Problems}}
Our review here is by no means exhaustive and the interested reader may refer to the  reviews~\cite{arridge_solving_2019,ongie_deep_2020} (among others). A natural, yet naive, way to solve \eqref{eq:prob_inv} is to learn from pairs of $(\xvc,\yv)$ a neural network that approximates an analytic "inverse" to the forward operator $\fop$. While this approach can provide qualitatively satisfactory results, it does not take into account explicitly the physics of the problem (the forward model \eqref{eq:prob_inv}), and lacks in particular data consistency. This approach lacks a deep understanding of its recovery guarantees with the only exception of the recent work of \cite{Petersen23} who provided a generalization bound, which is motivated by a machine learning perspective rather than an inverse problem one. To overcome some of these shortcomings, the dominant state-of-the-art approach is hybrid, and consists in mixing model- and data-driven methods to get the best of both worlds. There exists a vast array of such hybrid methods among which the most prominent are Plug-and-Play (PnP, see the review in \cite{kamilov2023plug}), learned regularization of a variational problem~\cite{prost_learning_2021}, and ``unrolling'' or ``unfolding'' methods (see the review \cite{monga_algorithm_2021}). While PnP uses a denoiser network to restrict the range of acceptable signals, one could restrict the set of possible signals to the range of a generative model (see the survey in e.g., \cite{duff2024regularising}). When no or not enough data is available, a well known alternative is the DIP framework~\cite{ulyanov_deep_2020}, and its variants~\cite{liu_image_2019,mataev_deepred_2019,shi_measuring_2022,zukerman_bp-dip_2021,Tirer24}. However, we are not aware of any theoretical work on recovery performance of the DIP except our previous work \cite{buskulic2024convergenceJournal,buskulic2024descentarxiv}. Our aim in this paper is to complement these results when momentum-based accelerated algorithms are used for training. 

\noindent{\textbf{Implicit regularization, Training Dynamics and Overparametrization}}
Neural networks are very high-dimensional non-linear parametric functions that are optimized/trained to minimize a given loss function. This should lead to highly non-convex optimization problems that are known to be challenging due to possibly many local minima and saddle points. Even more so in the context of inverse problems. In fact, even if the neural network is complex enough (overparametrized) to ensure zero empirical error, the set of minimizers may be large. Therefore, it may very well be the case that some minimizers are better than others (e.g. generalize, are stable, etc.). Optimization algorithms such as gradient descent introduce a bias in this choice: an iterative method is biased towards certain solutions of the problem it solves and thus may converge to a solution with certain properties. Since this bias is a by-product rather than an explicitly enforced property, it is known in the literature as {\textit{implicit regularization}}. This clearly highlights the importance of the optimization algorithm as implicit regularizer, and has played an important role in understanding either statistical learning guarantees of such implicit regularization \cite{bartlett_deep_2021,fang_mathematical_2021}, or the role of implicit regularization for inverse problems~\cite{KaltenbacherIterative08}. Understanding the role and implications of implicit regularization of an iterative algorithm for learning neural networks to solve inverse problems is at the heart of this chapter.

The modern approach to convergence of neural network training is based on gradient dominated inequalities from which one can deduce by simple integration an exponential convergence of the gradient flow to a zero-loss solution. This allows to obtain convergence guarantees for networks trained to minimize a mean square error by gradient flow~\cite{chizat_lazy_2019} or gradient descent~\cite{du_gradient_2019,arora_fine-grained_2019, oymak_overparameterized_2019, oymak_toward_2020}. Recently, it has been found that some kernels play a very important role in the analysis of convergence of the gradient flow when used to train neural networks. In particular the semi-positive definite kernel given by $\Jgt\Jgt\tp$, where $\Jgt$ is the Jacobian of the network at time $t$. When all the layers of a network are trained, this kernel is a combination of the \textit{Neural Tangent Kernel} (NTK)~\cite{jacot_neural_2018} and the Random Features Kernel (RF)~\cite{rahimi_weighted_2008}. The goal is then to control the eigenvalues of the kernel to ensure that they remain bounded away from zero, which entails convergence to a zero-loss solution at an exponential rate. The control of the eigenvalues of the kernel is done through a random initialization and the overparametrization of the network. This is also closely related to the celebrated Hartman-Grobman theorem in dynamical systems. 

However, these works do not account for the inverse problem setting. Moreover, they only study the gradient flow or gradient descent while inertia/momentum-based algorithms are dominant now. Thus there is a clear need for an analysis targetting recovery guarantees of the DIP method for inverse problems by properly accommodating for the forward operator.

\noindent{\textbf{Inertia-based Optimization}}
A large body of literature has been devoted to studying inertial optimization methods that we do not review for obvious space limitation. In the seminal work of Polyak~\cite{polyak_methods_1964}, he proposed the HBF system (i.e., setting $\beta=0$ in \eqref{eq:DIN}) which achieves exponential convergence for strongly convex smooth functions with an optimal convergence rate when $\alpha$ is chosen as the square-root of the strong convexity modulus. This system is however no faster than the gradient flow for the non-strongly convex case. It is also known that HBF may suffer traverse oscillations which motivated the introduction of Hessian damping~\cite{alvarez2002second}. Note that the Hessian damping term appears as the derivative of the gradient with respect to time, which opens the door to first-order optimization algorithms after proper discretization. System \eqref{eq:DIN} and its discretizations have been thoroughly studied in the convex and strongly convex case where $\alpha$ is an asymptotically vanishing viscous damping coefficient, see~\cite{attouch2022first}. In the nonconvex case, \eqref{eq:DIN} was studied in \cite{Maulen24} and \cite{casterainertial} with very promising performance when applied to neural network training. Our work brings together optimization results for inertial dynamics with overparametrization to obtain recovery results of the DIP method when solving linear inverse problems.



%
\section{Continuous-time Setting}\label{sec:continuous}

We will first analyze the trajectory of the parameters of a network trained through~\eqref{eq:DIN} as a continuous dynamical system. We start by showing that it is a well-posed system and then present our results showing accelerated convergence guarantees (and an associated recovery bound) compared to the gradient flow case for the right choice of $(\alpha,\beta)$. We also provide an overparametrization bound under which a two-layer network benefits from these guarantees. We will work under the following assumptions:
\begin{mdframed}[]
    \begin{assumption}\label{ass:l_MSE}
        $\lossy$ is the MSE loss, i.e., $\lossy(\zv) = \frac{1}{2}\norm{\zv - \yv}^2$.
    \end{assumption}
    \begin{assumption}\label{ass:phi_first_diff}
       $\phi \in \cC^1(\R)$ and $\exists B > 0$ such that $\sup_{x \in \R}|\phi'(x)| \leq B$ and $\phi'$ is $B$-Lipschitz continuous.
    \end{assumption}
\end{mdframed}
\medskip

Note also that the MSE loss allows to easily link the loss to its gradient. The MSE case is widely used and we refrain from extending our results to a more general class of KL smooth losses as in \cite{buskulic2024convergenceJournal,buskulic2024descentarxiv} to avoid unnecessary technicalities. The above two assumptions ensure that $\thetav \mapsto \nabla_{\thetav} \lossy(\fop\gv(\uv,\thetav))$ is locally Lipschitz continuous. This will be important when studying local well-posedness of \eqref{eq:DIN}. 

\subsection{Well-Posedness}
When $\beta>0$, the second-order dynamical system given in~\eqref{eq:DIN} can be equivalently formulated as a first-order system both in time and space. We adapt the results given in~\cite{attouch_effect_2023} to show this equivalence.

\begin{theorem}\label{thm:equivalence_second_and_first_order}
Suppose that $\alpha\geq 0$ and $\beta>0$. Then the following statement are equivalent:
    \begin{enumerate}
        \item\label{bull:second_order} $\thetav:[0,+\infty[\to\R^p$ is a solution trajectory of~\eqref{eq:DIN} with the initial conditions $\thetav(0) = \thetavz$ and $\dot{\thetav}(0) = \dot{\thetav}_0$.
        \item\label{bull:first_order} $(\thetav,\qv):[0,+\infty[ \to \R^p\times\R^p$ is a solution trajectory of the first-order system
        \begin{align}\label{eq:first_order}
            \begin{cases}
                \thetavdott + \beta \gradlosstheta -\pa{\frac{1}{\beta} - \alpha}\thetavt + \frac{1}{\beta}\qvt &= 0 \\
                \dot{\qv}(t) - \pa{\frac{1}{\beta} - \alpha}\thetavt + \frac{1}{\beta}\qvt &= 0 
            \end{cases}
        \end{align}
        with initial conditions $\thetav(0) = \thetavz$ and $\qv(0)=\qvz=-\beta\pa{\dot{\thetav}_0+\beta\nabla_{\thetav}\lossyzy}+\pa{1 - \alpha\beta}\thetavz$.
    \end{enumerate}
\end{theorem}

\begin{proof}
    $\ref{bull:first_order}\implies\ref{bull:second_order}.$ We start by differentiating the first equation of~\eqref{eq:first_order} which gives
    \[
    \thetavddott +\beta\derivgradlosstheta - \pa{\frac{1}{\beta} - \alpha}\thetavdott + \frac{1}{\beta}\dot{\qv}(t) = 0.
    \]
    We replace $\dot{\qv}(t)$ by using the second line of~\eqref{eq:first_order} and obtain that
    \[
    \thetavddott +\beta\derivgradlosstheta - \pa{\frac{1}{\beta} - \alpha}\thetavdott + \frac{1}{\beta}\pa{\pa{\frac{1}{\beta} - \alpha}\thetavt - \frac{1}{\beta}\qvt} = 0.
    \]
    Now we replace $\qvt$ by its expression from the first line of~\eqref{eq:first_order} and get
    \[
    \thetavddott +\beta\derivgradlosstheta - \pa{\frac{1}{\beta} - \alpha}\thetavdott + \frac{1}{\beta}\pa{\thetavdott + \beta\gradlosstheta} = 0.
    \]
    Once simplified, we obtain \eqref{eq:DIN}.
    The initial conditions are directly transferable as both $\thetav(0)$ and $\dot{\thetav}_0$ are defined the same way in both~\eqref{eq:DIN} and \eqref{eq:first_order}

    $\ref{bull:second_order}\implies\ref{bull:first_order}.$ Denoting $\qvt = \beta\pa{-\thetavdott - \beta\gradlosstheta + \pa{\frac{1}{\beta}-\alpha}\thetavt}$ and differentiating, we get that
    \[
    \dot{\qv}(t) = \beta\pa{-\thetavddott -\beta\derivgradlosstheta + \pa{\frac{1}{\beta} - \alpha}\thetavdott}.
    \]
    In view of $\thetavddott$ in~\eqref{eq:DIN}, we obtain that
    \[
    \dot{\qv}(t) = \thetavdott + \beta\gradlosstheta .
    \]
    By rearranging the terms and the definition of $\qvt$, we obtain both expressions of~\eqref{eq:first_order}. Furthermore, replacing in $\qv(0)$ the initial conditions given in~\eqref{eq:DIN} gives the initial conditions of~\eqref{eq:first_order} concluding the proof.
\end{proof}

Theorem~\ref{thm:equivalence_second_and_first_order} is valid for any initial condition $\dot{\thetav}(0)$, which includes the special case of~\eqref{eq:DIN} where $\dot{\thetav}(0) = \mathbf{0}$. Therefore, from now on, and without loss of generality, we will take $\dot{\thetav}(0) = \mathbf{0}$. The reason for this choice will be transparent later. Thanks to this first-order reformulation, we will be able to invoke the Cauchy-Lipschitz theorem to show the existence and uniqueness of a solution of our original system. Towards this goal, we write \eqref{eq:first_order} in the compact form
\begin{align}\label{eq:condensed_first_order}
    \begin{cases}
        \dot{\zv}(t) + \nabla G\pa{\zv(t)} + D\pa{\zv(t)} = 0\\
        \zv(0)=\pa{\thetavz,-\beta\pa{\beta\nabla_{\thetav}\lossyzy}+\pa{1 - \alpha\beta}\thetavz} ,
    \end{cases}
\end{align}
where $\zv(t) = (\thetavt,\qvt) \in \R^p\times\R^p$, $G:\R^p\times\R^p \mapsto \pa{\beta\gradlosstheta,\mathbf{0}} \in \R^p\times\R^p$ and $D:\R^p\times\R^p\to\R^p\times\R^p$ is given by
    \[
    D(\zv(t)) = \pa{-\pa{\frac{1}{\beta}-\alpha}\thetavt + \frac{1}{\beta}\qvt, -\pa{\frac{1}{\beta}-\alpha}\thetavt + \frac{1}{\beta}\qvt}.
    \]
Equipped with this condensed form we can show that~\eqref{eq:condensed_first_order} is well-posed and thus so is~\eqref{eq:DIN}. We start by defining our notion of solution. 
\begin{definition}\label{def:strongsol}
For $T > 0$, we will say that $\thetav: t \in [0,T] \to \R^p$ is a strong solution of \eqref{eq:DIN} on $[0,T]$ if the following holds:
\begin{enumerate}[label=$\bullet$]
\item $\thetav(\cdot) \in \cC^0([0,T])$;
\item $\thetav(\cdot) \in \cC^1$ on every compact set of the interior of $[0,T[$; 
\item $\dot{\thetav}(\cdot)$ is absolutely continuous on every compact set of the interior of $[0,T[$;
\item \eqref{eq:DIN} holds for almost all $t \in ]0,T[$.
\end{enumerate}
A trajectory $\thetav: t \in [0,+\infty[ \to \R^p$  is a strong global solution of \eqref{eq:DIN} if it is a strong solution on $[0, T]$ for any $T > 0$.
\end{definition}

\medskip

\begin{proposition}\label{prop:cauchy_lip_condensed}
Assume that \ref{ass:l_MSE}-\ref{ass:phi_first_diff} hold and $\alpha \geq 0$ and $\beta \geq 0$. Then there exists $T(\thetavz)\in[0,+\infty[$ and a unique strong solution trajectory $\thetav(\cdot)$ of~\eqref{eq:DIN} on $[0,T(\thetavz)]$. 
\end{proposition}
\begin{proof}
Let us start with the case $\beta > 0$. We know by our assumptions and standard differential calculus on $\gv(\uv,\cdot)$ that $\gradlosstheta$ is locally Lipschitz continuous. Furthermore, the affine operator $D$ is itself globally Lipschitz. Then by the Cauchy-Lipschitz Theorem~\cite[Theorem~0.4.1]{Haraux91}, we obtain that \eqref{eq:condensed_first_order} has a unique maximal solution $\zv(\cdot) \in \cC^0([0,T(\thetavz)])$, where the dependence is only on the initial condition $\thetavz$ as we took $\dot{\thetav}_0=0$. Moreover, $\zv(\cdot) \in \cC^1$ on every compact set of the interior of $[0,T(\thetavz)[$. This gives us the first item thanks to Theorem~\ref{thm:equivalence_second_and_first_order}. Since $\thetav \in \cC^1$ on every compact set of the interior of $[0,T(\thetavz)[$ and $\nabla_{\thetav} \lossy(\fop\gv(\uv,\cdot))$ is locally Lipschitz continuous, we get that $t \mapsto \nabla_{\thetav} \lossy(\fop\gv(\uv,\thetav(t)))$ is Lipschitz continuous, hence absolutely continuous, on every compact set of the interior of $[0,T(\thetavz)[$. The second and third claims then follow from the first line of~\eqref{eq:first_order}.

For the case $\beta=0$, we use the standard equivalent first-order system of \eqref{eq:DIN} in phase-space (position-velocity) by introducing the velocity variable $\vvb(t) = \dot{\thetav}(t)$. The same reasoning as above leads the claim.
\end{proof}

The time $T(\thetavz)$ is known as the maximal existence time of the solution. By the blow-up alternative, either $T(\thetavz) = +\infty$ and in that case we say that the solution is \textit{global}, or $T(\thetavz) < +\infty$ and the solution blows-up in finite time i.e., $\norm{\thetavt}\to+\infty$ when $t\to T(\thetavz)$. In fact, thanks to Lemma~\ref{lemma:theta_summable} and Lemma~\ref{lemma:link_params_singvals} to be stated and proved later, we can show that the local strong solution is actually global provided that $\alpha$ and $\beta$ are well-chosen and the dynamic is well initialized.

\begin{proposition}\label{prop:global_solution}
Assume that \ref{ass:l_MSE}-\ref{ass:phi_first_diff} hold, $\alpha > 0$ and $0 < \beta < 2/\alpha$. Suppose also that \eqref{eq:bndR} is verified. Then \eqref{eq:DIN} has a unique global strong solution.
\end{proposition}
\begin{proof}
By Proposition~\ref{prop:cauchy_lip_condensed}, we know that there exists a unique strong maximal solution to~\eqref{eq:DIN}. Following the above discussion, it is sufficient to show that $\thetav(\cdot)$ is bounded. This follows from Lemma~\ref{lemma:theta_summable} and Lemma~\ref{lemma:link_params_singvals}\ref{claim:sigval_bounded_everywhere}.
\end{proof}

\subsection{Convergence and Recovery Guarantees}\label{sec:main_results}

We now state in the next theorem how, if a given network obeys some condition on its initialization and is trained with~\eqref{eq:DIN}, we obtain accelerated convergence guarantees of the loss and the network parameters to a zero loss solution, with respect to the guarantees obtained with gradient flow. We also give the associated accelerated early-stopping bound and signal convergence bound.

\begin{theorem}\label{thm:main}
Assume that \ref{ass:l_MSE}-\ref{ass:phi_first_diff} hold. Let $\thetav(\cdot)$ be a solution trajectory of~\eqref{eq:DIN} with $\alpha=\sigminjgz\sigminA$ and $\beta=\frac{1}{2\alpha}$ where the initialization $\thetavz$ is such that
    \begin{equation}\label{eq:bndR} 
    \sigminjgz > 0 \qandq R' < R ,
    \end{equation}
    where $R'$ and $R$ obey
    \begin{equation}\label{eq:RandR'}
    R' = \eta\sqrt{\xi\lossyzy} \qandq R = \frac{\sigminjgz}{2\Lip_{\Ball(\thetavz,R)}(\Jg)} 
    \end{equation}
    with
    \[
    \xi = 1 + \frac{\kappa(\Jgz)^2\kappa(\fop)^2}{4} \qandq
    \eta = \frac{4\max\pa{\sigminjgz\sigminA,\frac{1+\sqrt{2}}{2}}}{\min\pa{{\sigminjgz^2\sigminA^2}, \frac{3}{4}}}.
    \]
    Then, the following holds:
    \begin{enumerate}[label=(\roman*)]
    \item \label{thm:item_loss_convergence} the loss converges to $0$ at the rate
    \begin{align}\label{eq:lossrate_continuous}
    \lossyty &\leq \xi\lossyzy\exp\pa{-\frac{\sigminjgz\sigminA}{2}t} .
    \end{align}
    Moreover, $\thetavt$ converges to a global minimizer $\thetav_{\infty}$ at the rate
    \begin{align}\label{eq:thetarate}
        \norm{\thetavt - \thetav_{\infty}} \leq \eta\sqrt{\xi\lossyzy}\exp\pa{-\frac{\sigminjgz\sigminA}{4}t}.
    \end{align}
    \item\label{thm:item_overfitting} We have
    \begin{align}\label{eq:yrate}
        \norm{\yvt - \yvc}\leq2\norm{\veps} \quad \text{when} \quad t\geq\frac{4}{\sigminjgz\sigminA}\ln\pa{\frac{\sqrt{2\xi\lossyzy}}{\norm{\veps}}}.
    \end{align}
    \item \label{thm:item_signal_convergence} If, moreover,
    \begin{assumption}\label{ass:A_inj}
    $\ker{(\fop)} \cap T_{\Sigma'}(\xvc_{\Sigma'}) = \{0\}$ with $\Sigma'\eqdef \Sigma_{\Ball_{R' + \norm{\thetavz}}}$,
    \end{assumption}
    then
    \begin{align}\label{eq:xrate}
    \begin{split}
        \norm{\xvt - \xvc} \leq\frac{\sqrt{2\xi\lossyzy}\exp\pa{-\frac{\sigminjgz\sigminA}{4}t}}{\lmin(\fop;T_{\Sigma'}(\xvcsigma))} + \frac{\norm{\veps}}{\lmin(\fop;T_{\Sigma'}(\xvcsigma))} \\
        + \pa{1 + \frac{\norm{\fop}}{\lmin(\fop;T_{\Sigma'}(\xvcsigma))}}\dist(\xvc,\Sigma').
    \end{split}
    \end{align}
    \end{enumerate}
\end{theorem}

\begin{proof}
    See Section~\ref{subsec:proof_main_continuous}
\end{proof}

\subsection{Discussion and consequences}\label{sec:discussions_continuous}


{\textbf{Role of $\alpha$ and accelerated rate}~}
The first result of our theorem shows that if the network training is well-initialized, i.e. according to \eqref{eq:bndR}, and with an appropriate choice of $\alpha$ and $\beta$, the network weights will converge to a zero-loss solution and the loss decreases at an exponential rate that depends on $\sigminjgz\sigminA$. It was shown in \cite[Theorem~3.2]{buskulic2024convergenceJournal} that training with gradient flow has also exponential convergence at the rate $O\pa{\exp(-\sigminjgz^2\sigminA^2 t/4)}$. Compared to the latter, our rate in Theorem~\ref{thm:main} is provably accelerated in the ill-conditioned case. This is expected as it is known for optimization dynamics featuring inertia in the smooth and strongly convex case when $\alpha$ is appropriately tuned as the square-root of the strong convexity modulus~\cite{polyak_methods_1964,attouch2022first}. This rate is known to be optimal \cite{nesterov2013introductory} for this class of objectives. Our setting is of course more intricate and general as our problem is nonconvex. One also observes that the effect of acceleration depends on the conditioning of the forward operator and specifically, the worse the conditioning, the better the acceleration. However, whereas in the gradient flow case the multiplying constant in the rate depends solely on $\lossyzy$, the extra-term $\xi$ in the constant in the rates of Theorem~\ref{thm:main} reveals a quadratic dependence on the condition numbers $\kappa(\Jgz)$ and $\kappa(\fop)$. This is a mild price to pay compared to the exponential gain in the rate. 

The viscous and geometric damping parameters $\alpha$ and $\beta$ were optimized to achieve the (optimal) accelerated exponential rate. However, one has to keep in mind that this rate only holds in the initialization regime where \eqref{eq:bndR} is true (which will turn out to hold in the overparameterized regime as we will show in the forthcoming section). Since both $R'$ and the bound~\eqref{eq:thetarate} on the convergence of the network parameters grow linearly with $\eta$, it is tempting to make $\eta$ as small as possible by adjusting $\alpha$ and $\beta$ as $\eta$ clearly depends on them, see~\eqref{eq:eta}. However, minimizing $\eta$ in such a way should be done without harming the exponential convergence rate, particularly in the ill-conditioned setting i.e. $\sigminA$ small. In fact, minimizing \ref{eq:eta} in $\alpha$ and $\beta$ suggests to take $\beta=1/\alpha$ and $\alpha$ as a constant. In turn, from~\eqref{eq:convrate_any_alpha_beta}, the convergence rate would be $O(1)$, which is vacuous. Even choosing $\beta$ arbitrarily close to, but strictly less than, $1/\alpha$, would result in a rate $O\pa{\exp(-c\sigminjgz^2\sigminA^2 t)}$, for some $c > 0$, when $\sigminA$ is small. This is the same rate as the gradient flow which cancels out the advantage brought by inertia.  On the other hand, our choice of $\alpha$ and $\beta$ leads to the (optimal) accelerated rate, but does not minimize $\eta$, which will scale as $O(\sigminjgz^{-2}\sigminA^{-2})$ for the ill-conditioned case. $\eta$ would scale as $O(\sigminjgz^{-1}\sigminA^{-1})$ when minimizing it in $\alpha$ and $\beta$.

\noindent{\textbf{Early stopping}}
While \eqref{eq:lossrate_continuous} ensures convergence to a zero-loss solution, it does so by overfitting the noise inherent to the inverse problem. A classical way to avoid this is to use an early stopping strategy, hence ensuring that the solution in the observation space will lie in a ball around the sought after observation $\yvc$. This is precisely what \eqref{eq:yrate} states. It is worth mentioning that early stopping has been used by practitioners of the DIP model trained with gradient descent and our results give this intuition firm theoretical grounds. In view of our discussion on the rate accelerated above, it is clear that our early stopping bound is much better than that of \cite[Theorem~3.2]{buskulic2024convergenceJournal} for gradient flow.

\noindent{\textbf{Signal recovery}}
Similarly to the case of gradient flow, see \cite[Theorem~3.2]{buskulic2024convergenceJournal}, our recovery bound on $\xvc$ in \eqref{eq:xrate} is the sum of three terms. The last two ones correspond respectively to the ``noise error'' inherent to the forward model, and the ``modeling error'' which captures the expressivity of the trained network, i.e. its ability to generate solutions close to $\xvc$. These two terms are exactly the same as those in \cite[Theorem~3.2]{buskulic2024convergenceJournal}. The first term \eqref{eq:xrate} is an ``optimization error''. This is where the role of inertia is important and this error in our case is much smaller as discussed above.

The bounds in \eqref{eq:xrate} depend on the minimal conic singular value $\lmin(\fop;T_{\Sigma'}(\xvcsigma))$ which is bounded away from zero thanks to the restricted injectivity condition \eqref{ass:A_inj}. This is a classical and minimal assumption in the inverse problem literature if one hopes for recovering $\xvc$ even in the noiseless case. Assuming the rows of $\fop$ are linearly independent, one easily checks that \eqref{ass:A_inj} imposes that $m \geq \dim(T_{\Sigma'}(\xvcsigma))$. As it was also observed in \cite{buskulic2024convergenceJournal}, there is a trade-off between the restricted injectivity condition \eqref{ass:A_inj} and the expressivity of the network. If the model is highly expressive then $\dist(\xvc,\Sigma')$ will be smaller. But this is likely to come at the cost of making $\lmin(\fop;T_{\Sigma'}(\xvcsigma))$ decrease, as restricted injectivity may be required to hold on a larger subset (cone). Although this observation is to be tempered as the dimension $T_{\Sigma'}(\xvcsigma)$ does not necessarily increase as $\Sigma'$ gets larger. This discussion relates with the work on the instability phenomenon observed in learned reconstruction methods \cite{antun2020instabilities,gottschling2020troublesome}. In fact, one fundamental problem that creates these instabilities in the reconstruction is that $\ker(\fop)$ can be non-trivial. The restricted injectivity condition guarantees stable reconstruction but the error bound degrades with decreasing $\lmin(\fop;T_{\Sigma'}(\xvcsigma))$.


\subsection{Wide Two-Layer DIP Network}\label{sec:dip}

It is now natural to ask when a network obeys~\eqref{eq:bndR} and thus enjoys the convergence and reconstruction guarantees of Theorem~\ref{thm:main}. Our way of ensuring this is to be in a sufficiently overparametrized regime; see Section~\ref{sec:prior} for a review of the role of overparametrized when training neural networks. Informally, the question pertains to determining, for a network architecture and a random initialization, the number of neurons or parameters of the network to ensure the validity of~\eqref{eq:bndR} with high probability. Indeed, good statistical properties arise from overparametrized networks, enabling control over the eigenspace of the network Jacobian at initialization. Similar to other related works, we will primarily focus on studying shallow networks. Extensions to deeper network are beyond the scope of this chapter.

Recall that in our self-supervised DIP setting, the input $\uv$ is sampled randomly and fixed during training. The network is then trained to map $\uv$ to a signal $\xv$ such that $\fop\xv$ is close to $\yv$. We use a one-hidden layer network by taking $L=2$ in Definition~\ref{def:nn}, which we write as
\begin{align}\label{eq:dipntk}
    \gdip = \frac{1}{\sqrt{k}}\Vv\phi(\Wv\uv)
\end{align}
with  $\Vv \in \R^{n \times k}$ and $\Wv \in \R^{k \times d}$, and $\phi$ an element-wise nonlinear activation function. To establish our overparametrization bound, we will impose the following assumptions where $\Cphi = \sqrt{\Expect{X\sim\stddistrib}{\phi(X)^2}}$ and $\Cphid = \sqrt{\Expect{X\sim\stddistrib}{\phi'(X)^2}}$\footnote{Observe that $\Cphid \leq B$ under \ref{ass:phi_first_diff}.}:
\begin{mdframed}[frametitle={Assumptions on the network input and initialization}]
\begin{assumption}\label{ass:u_sphere}
$\uv$ is a uniform vector on $\sph^{d-1}$;
\end{assumption}
\begin{assumption}\label{ass:w_init}
$\Wv(0)$ has iid entries from $\stddistrib$ and $\Cphi < +\infty$; 
\end{assumption}
\begin{assumption}\label{ass:v_init}
$\Vv(0)$ is independent from $\Wv(0)$ and $\uv$, and its entries are zero-mean independent $D$-bounded random variables of unit variance.
\end{assumption}
\end{mdframed}

These assumptions are quite standard for neural networks and very easy to verify. For~\ref{ass:v_init}, as an example, one can use iid entries chosen from the uniform distribution on a compact interval. We can now state our overparametrization bound under which~\eqref{eq:bndR} holds, and thus, so do the guarantees of Theorem~\ref{thm:main}. We denote the signal-to-noise ratio as $\SNR=\norm{\fop\xvc}/\norm{\veps}$.
\begin{theorem}\label{th:dip_two_layers_converge}
Suppose that assumptions \ref{ass:l_MSE} and~\ref{ass:phi_first_diff} hold. Consider the one-hidden layer network \eqref{eq:dipntk} where both layers are trained with the initialization satisfying \ref{ass:u_sphere} to \ref{ass:v_init} and the architecture parameters obeying
\begin{align*}
k \geq C(1+\kappa(\fop)^4)\frac{\max\pa{\sigminA^4,c_1}}{\min\pa{{\sigminA^8}, c_2}} n \pa{\norm{\fop}^4 n^2 + \norminf{\fop\xvc}^4\pa{1+ \SNR^{-1}}^4 m^2}.
\end{align*}
Then \eqref{eq:bndR} holds with probability at least $1 - 5e^{-(n-1)}- 2n^{-1}$. Here $c_1, c_2, C > 0$ are  absolute constants that depend only on $\Cphi, \Cphid, B$ and $D$.
\end{theorem}
\begin{proof}
See Section~\ref{subsec:dip_two_layers_converge}.
\end{proof}

The overparametrization bound scales as $k \geq n^3 + nm^2$, which is similar to gradient flow \cite[Theorem~4.1]{buskulic2024convergenceJournal}. However, as we discussed in Section~\ref{sec:discussions_continuous}, training with \eqref{eq:DIN} achieves an optimal exponential rate but at the price of the initialization condition which becomes more stringent as the conditioning of $\fop$ degrades. This is clearly reflected in our overparametrization bound. Indeed, in the extremely ill-conditioned case, we have an extra multiplying factor that scales as $\kappa(\fop)^4$ compared to~\cite[Theorem~4.1]{buskulic2024convergenceJournal}. Whether this can be improved to get the best of both worlds is an open question that we leave to a future work. 

We observe again that the set $\Sigma'$ on which \ref{ass:A_inj} is required to hold is random. Nevertheless, using similar arguments as in \cite[Remark~4.2]{buskulic2024convergenceJournal}, one can show that $\Sigma' \subset \Sigma_{\Ball_{\rho}(0)}$, where
\begin{equation*}
\rho \lesssim \frac{\max\pa{\sigminA,(c_1)^\frac{1}{4}}}{\min\pa{{\sigminA^2}, (c_2)^\frac{1}{4}}}(1+\kappa(\fop))\pa{\norm{\fop}\sqrt{n} + \norminf{\fop\xvc}\pa{1+ \SNR^{-1}}\sqrt{m}} + \sqrt{k}\pa{\sqrt{n}+\sqrt{d}}
\end{equation*}
with probability at least $1 - 5e^{-(n-1)} - 2e^{-kd} - 2n^{-1}$. In the overparametrized regime, $\rho$ scales as $O\pa{\sqrt{k}\pa{\sqrt{n}+\sqrt{d}}}$. This confirms the intuitively expected behaviour that expressivity of $\Sigma'$ is better as the overparametrization increases.

\subsection{Proofs}\label{sec:proofs}

\subsubsection{Proof of Theorem~\ref{thm:main}}\label{subsec:proof_main_continuous}

We will start by showing some intermediate lemmas necessary to prove our main theorem. For these proofs, we will use the following Lyapunov function given in the original work of~\cite{alvarez2002second}:
\begin{align}\label{eq:lyapunov}
    V(t) = \lossyty + \frac{1}{2}\norm{\thetavdott + \beta\gradlosstheta}^2.
\end{align}
We prove in the following lemma that $V(t)$ converges which is then used in the proof of Proposition~\ref{prop:global_solution} to obtain that $\thetav$ is a global solution of~\eqref{eq:DIN}.

\begin{lemma}\label{lemma:thetadotL2}
Assume that \ref{ass:l_MSE}-\ref{ass:phi_first_diff} hold, $\alpha > 0$ and $0 \leq \beta \leq \frac{2}{\alpha}$. Let $\thetav(\cdot)$ be a solution trajectory of~\eqref{eq:DIN}. Then,
    \begin{enumerate}[label=(\roman*)]
    \item $V(t)$ is nonincreasing and converges.\label{lemma:thetadotL2claim1}
    \item $\dot{\thetav}(\cdot)\in L^2([0,+\infty[)$. If $\beta \in ]0,2/\alpha[$, then $\nabla_{\thetav} \lossy(\yv(\cdot))\in L^2([0,+\infty[)$.\label{lemma:thetadotL2claim2}
    \item If $\thetav(\cdot)$ is bounded and $\beta \in ]0,2/\alpha[$, then $\lim_{t\to+\infty}\norm{\gradlosstheta}=0$.\label{lemma:thetadotL2claim3}
    \end{enumerate}
\end{lemma}

\begin{proof}
    We start by differentiating the Lyapunov function $V(t)$ and obtain that
    \begin{align}
        &\Vdott \leq \dotprod{\gradlosstheta}{\thetavdott} + \dotprod{\thetavdott+\beta\gradlosstheta}{\thetavddott + \beta\derivgradlosstheta} \nonumber\\
        &\leq \dotprod{\thetavdott}{\gradlosstheta + \thetavddott+\beta\derivgradlosstheta} + \beta\dotprod{\gradlosstheta}{\thetavddott+\beta\derivgradlosstheta} \nonumber.
    \end{align}
    We now replace $\thetavddott + \beta\derivgradlosstheta$ by
using~\eqref{eq:DIN} which gives
    \begin{align*}
        \Vdott \leq -\alpha\norm{\thetavdott}^2 - \beta\norm{\gradlosstheta}^2 + \beta\alpha\dotprod{\gradlosstheta}{-\thetavdott}.
    \end{align*}
    Applying Young's inequality we get
    \begin{align}\label{eq:lyapu_deriv}
        \Vdott &\leq -\alpha\norm{\thetavdott}^2 - \beta\norm{\gradlosstheta}^2 + \frac{\beta^2\alpha}{2}\norm{\gradlosstheta}^2 + \frac{\alpha}{2} \norm{\thetavdott}^2 \nonumber\\
        &\leq -\frac{\alpha}{2}\norm{\thetavdott}^2 - \beta\pa{1 - \frac{\beta\alpha}{2}}\norm{\gradlosstheta}^2.
    \end{align}
    We get claim~\ref{lemma:thetadotL2claim1} as $V$ is nonnegative and given the choice of $\beta$. Integrating \eqref{eq:lyapu_deriv}, we also obtain claim~\ref{lemma:thetadotL2claim2}.

By our assumptions, we know that $\nabla_{\thetav}\lossy(\fop\gv(\uv,\cdot))$ is locally Lipschitz continuous. By the boundedness assumption, we have that $\nabla_{\thetav} \loss(\yv(\cdot))\in L^\infty([0,+\infty[)$. Moreover, since $\dot{\thetav}(\cdot) \in L^2([0,+\infty[)$ and is continuous, then $\dot{\thetav}(\cdot) \in L^\infty([0,+\infty[)$. These facts imply that there exists $L > 0$ such that for every $s,t \geq 0$
\begin{align*}
|\norm{\nabla_{\thetav} \loss(\yv(t))}^2 - \norm{\nabla_{\thetav} \loss(\yv(s))}^2| 
&\leq 2\sup_{\tau \geq 0}\norm{\nabla_{\thetav} \loss(\yv(\tau))} \norm{\nabla_{\thetav} \loss(\yv(t)) - \nabla_{\thetav} \loss(\yv(s))} \\
&\leq 2L\sup_{\tau \geq 0}\norm{\nabla_{\thetav} \loss(\yv(\tau))} \norm{\thetav(t) - \thetav(s)} \\
&\leq 2L\sup_{\tau \geq 0}\norm{\nabla_{\thetav} \loss(\yv(\tau))}\sup_{u \geq 0}\norm{\dot{\thetav}(u)} |t-s| ,
\end{align*}
and thus $\norm{\nabla_{\thetav} \loss(\yv(\cdot))}^2$ is uniformly continuous. Since it is also integrable, Barbălat lemma yields claim~\ref{lemma:thetadotL2claim3}.

\end{proof}

\begin{lemma}
\label{lemma:theta_summable}
Assume that \ref{ass:l_MSE} and \ref{ass:phi_first_diff} hold, $\alpha > 0$ and $0 < \beta < \frac{2}{\alpha}$. Let $\thetav(\cdot)$ be a solution trajectory of \eqref{eq:DIN}. If for all $t \geq 0$, $\sigminjgt \geq \frac{\sigminjgz}{2} > 0$, then $\dot{\thetav}(\cdot) \in L^1([0,+\infty[)$. In turn, $\lim_{t \to +\infty}\thetav(t)$ exists.
\end{lemma}

\begin{proof}
By assumption, we have for all $t>0$ that
    \begin{align}
        \norm{\gradlosstheta}^2 &= \norm{\Jg(t)\tp\fop\tp\pa{\yvt - \yv}}^2 
        \geq 2{\sigminjgz^2\sigminA^2}\lossyty  \label{eq:grad_loss_MSE_sigmin}
    \end{align}
    where, in the inequality, we used that $\yv(t)-\yv \in \ran{\fop} = \ker(\fop\tp)^\perp$. This argument will be used repeatedly though we will not specify it.
    Now if we proceed to our Lyapunov function $V(t)$, we observe that
    \begin{align}\label{eq:bound_lyapu}
        V(t) &= \lossyty + \frac{1}{2}\norm{\thetavdott +\beta\gradlosstheta}^2\nonumber\\
        &\leq \lossyty + \norm{\thetavdott}^2 + \beta^2\norm{\gradlosstheta}^2\nonumber\\
         {\small \eqref{eq:grad_loss_MSE_sigmin}}&\leq \norm{\thetavdott}^2 + \pa{\beta^2 + \pa{2\sigminjgz^{2}\sigminA^{2}}^{-1}}\norm{\gradlosstheta}^2\nonumber\\
         &\leq \max\pa{1,\beta^2+\pa{2\sigminjgz^{2}\sigminA^{2}}^{-1}}\pa{\norm{\thetavdott}^2 + \norm{\gradlosstheta}^2}.
    \end{align}
    We now look at $\deriv{V(t)^{1/2}}{t}$. Without loss of generality, we assume that $V(t) \neq 0$ as otherwise $V(s)=0$ for all $s \geq t$ (remember that $V$ is nonincreasing), and thus $\dot{\thetav}(s)=\nabla_{\thetav}\lossy(\yv(s))=0$ for all $s \geq t$, and there is nothing to prove. We have the following chain of inequalities
    \begin{align*}
        \deriv{}{t}\sqrt{V(t)} &= \frac{\dot{V}(t)}{2\sqrt{V(t)}}\\
        {\small \eqref{eq:lyapu_deriv}}    &\leq \frac{-\alpha/2\norm{\thetavdott}^2 - \beta\pa{1 - \frac{\beta\alpha}{2}}\norm{\gradlosstheta}^2}{2\sqrt{V(t)}}\\
        {\small \eqref{eq:bound_lyapu}}    &\leq -\frac{\min\pa{\alpha/2, \beta\pa{1 - \frac{\beta\alpha}{2}}}\pa{\norm{\thetavdott}^2+\norm{\gradlosstheta}^2}}{2\max\pa{1,\beta+\pa{\sqrt{2}\sigminjgz\sigminA}^{-1}}\pa{\norm{\thetavdott}^2+\norm{\gradlosstheta}^2}^{1/2}} \\
        &\leq -\eta\inv \pa{\norm{\thetavdott}^2+\norm{\gradlosstheta}^2}^{1/2} ,
    \end{align*}
    where we let 
    \begin{equation}\label{eq:eta}
    \eta = \frac{2\max\pa{1,\beta+\pa{\sqrt{2}\sigminjgz\sigminA}^{-1}}}{\min\pa{\alpha/2, \beta\pa{1 - \frac{\beta\alpha}{2}}}} .
    \end{equation}
    Integrating, we get
    \begin{equation*}\label{eq:theta_bounded}
    \begin{aligned}
        \int^t_0\norm{\dot{\thetav}(s)}\mathrm{d}s &\leq \int^t_0 \pa{\norm{\dot{\thetav}(s)}^2+\norm{\nabla_{\thetav}\lossy(\yv(s))}^2}^{1/2}\mathrm{d}s \\
        &\leq -\eta\int^t_0 \deriv{V(s)^{1/2}}{s}\mathrm{d}s \\
        &\leq \eta \sqrt{V(0)} \\
        &\leq \eta \pa{\lossyzy + \frac{\beta^2}{2}\norm{\nabla_{\thetav}\lossyzy}^2}^{1/2} .
    \end{aligned}
    \end{equation*}
    where in the last inequality we used that $\dot{\thetav}(0) = \mathbf{0}$. In view of \ref{ass:l_MSE}, we have
\begin{align}\label{eq:bound_gradJg_init}
    \norm{\nabla_{\thetav}\lossyzy} = \norm{\Jgz\tp\fop\tp(\yvt-\yv)}\leq \norm{\Jgz} \norm{\fop}\sqrt{2\lossyzy}.
\end{align}
	Combining this with \eqref{eq:theta_bounded}, we obtain
	\begin{equation}\label{eq:dtheta_bounded}
	\int^t_0\norm{\dot{\thetav}(s)}\mathrm{d}s \leq \eta \sqrt{\pa{1 + \beta^2\norm{\Jgz}^2\norm{\fop}^2}\lossyzy} .
	\end{equation}
    Passing to the limit, we get that $\dot{\thetav}(\cdot)\in L^1\pa{[0,+\infty[}$ and thus, $\lim_{t \to +\infty}\thetav(t)$ exists by applying Cauchy's criterion to
\[
\thetav(t) = \thetavz + \int_0^t \dot{\thetav}(s) \ds .
\]
\end{proof}

\begin{lemma}\label{lemma:link_params_singvals}
Assume that \ref{ass:l_MSE} and \ref{ass:phi_first_diff} hold, $\alpha=\sigminjgz\sigminA$ and $\beta=\frac{1}{2\alpha}$. Recall $R$ and $R'$ from \eqref{eq:RandR'}. Let $\thetav(\cdot)$ be a solution trajectory of \eqref{eq:DIN}.
    \begin{enumerate}[label=(\roman*)]
        \item \label{claim:singvals_bounded_if_params_bounded} If $\thetav \in \Ball_R(\thetavz)$ then
        \begin{align*}
            \sigmin(\jtheta) \geq \sigminjgz/2.
        \end{align*}
        
        \item  \label{claim:params_bounded_if_singvals_bounded} If for all $s \in [0,t] $, $\sigminjgs \geq \frac{\sigminjgz}{2}$ then 
            \begin{align*}
                \thetavt \in \Ball_{R'}(\thetavz) .
            \end{align*}
        
        \item \label{claim:sigval_bounded_everywhere}
        If $R'<R$, then for all $t \geq 0$, $\sigminjgt \geq \sigminjgz/2$.
\end{enumerate}
\end{lemma}

\begin{proof}
    \begin{enumerate}[label=(\roman*)]
    \item Similar to~\cite[Lemma~3.11(i)]{buskulic2024convergenceJournal}.
    \item Using \eqref{eq:dtheta_bounded}, we have for $t > 0$
    \begin{align*}
        \norm{\thetavt - \thetavz} \leq \int^t_0\norm{\dot{\thetav}(s)}\mathrm{d}s \leq \eta \sqrt{\pa{1 + \beta^2\norm{\Jgz}^2\norm{\fop}^2}\lossyzy} .
    \end{align*}
    Replacing $\alpha$ and $\beta$ in \eqref{eq:eta} by their values according to our choice, the rhs of the last inequality is precisely $R'$, whence we get the claim.
    \item Similar to~\cite[Lemma~3.11(iii)]{buskulic2024convergenceJournal}.
    \end{enumerate}
\end{proof}

\begin{proof}[Proof of Theorem~\ref{thm:main}]
    \begin{enumerate}[label=(\roman*)]
    \item We follow a standard Lyapunov analysis. By Jensen's inequality,
    \[
    -\norm{\thetavdott}^2 \leq -\frac{1}{2}\norm{\thetavdott + \beta\gradlosstheta}^2 + \beta^2\norm{\gradlosstheta}^2 .
    \]
    Combining this with~\eqref{eq:lyapu_deriv} gives
    \begin{align}\label{eq:convrateode_any_alpha_beta}
        \Vdott&\leq - \frac{\alpha}{4}\norm{\thetavdott + \beta\gradlosstheta}^2 - \beta\pa{1 - \beta\alpha}\norm{\gradlosstheta}^2\nonumber\\
        {\small \eqref{eq:grad_loss_MSE_sigmin}} &\leq  - \frac{\alpha}{4}\norm{\thetavdott + \beta\gradlosstheta}^2 - 2\beta\pa{1 - \beta\alpha}{\sigminjgz^2\sigminA^2}\lossyty\nonumber\\
        &\leq -\min\pa{\frac{\alpha}{2},2\beta\pa{1 - \beta\alpha}{\sigminjgz^2\sigminA^2}}V(t).
    \end{align}
    Integrating, we obtain
    \begin{align}\label{eq:convrate_any_alpha_beta}
        \lossyty \leq V(t) \leq V(0)\exp\pa{-\min\pa{\frac{\alpha}{2},2\beta\pa{1 - \beta\alpha}{\sigminjgz^2\sigminA^2}}t}.
    \end{align}
    The optimal rate is obtained by setting $\beta=\frac{1}{2\alpha}$ and $\alpha=\sigminjgz\sigminA$ corresponding to the choice in the theorem, hence leading to
    \begin{align}\label{eq:loss_bounded_Vz}
        \lossyty \leq V(t) \leq V(0)\exp\pa{-\frac{\sigminjgz\sigminA}{2}t}.
    \end{align}
    By assumption we set $\dot{\thetav}(0)=\mathbf{0}$ which means that
    \[
    V(0) = \lossyzy + \frac{\beta^2}{2}\norm{\gradlossthetaz}^2 .
    \]
    From \eqref{eq:bound_gradJg_init}, we get that
    \[
    V(0) \leq \xi\lossyzy 
    \]
    where $\xi=1 + {\beta^2}\norm{\Jgz}^2\norm{\fop}^2$. Replacing $\beta$ with its value in the expression of $\xi$ and plugging into \eqref{eq:loss_bounded_Vz} concludes the proof of~\eqref{eq:lossrate_continuous}.

    By Lemma~\ref{lemma:theta_summable}, we know that $\thetav(\cdot)$ converges to some $\theta_{\infty}$. We use~\eqref{eq:loss_bounded_Vz} and a similar reasoning as for~\eqref{eq:theta_bounded} to obtain
    \begin{align*}
        \norm{\thetavt - \thetav_\infty} &\leq \int_t^{+\infty}\norm{\dot{\thetav}(s)}\mathrm{d}s\\
        &\leq \eta \sqrt{V(t)}\\
        &\leq \eta \sqrt{V(0)}\exp\pa{-\frac{\sigminjgz\sigminA}{4}t}\\
        &\leq \eta \sqrt{\xi\lossyzy}\exp\pa{-\frac{\sigminjgz\sigminA}{4}t}
    \end{align*}
    which shows~\eqref{eq:thetarate}.

    \item Continuity of $\fop$ and $\gv(\uv,\cdot)$ indicate that $\yv(\cdot)$ also converges to $\yv_\infty=\fop\gv(\uv,\thetav_{\infty})$. The early stopping bound can be obtained by using~\eqref{eq:lossrate_continuous}. Observe that
    \begin{align*}
        \norm{\yvt - \yvc} &\leq \norm{\yvt - \yv} + \norm{\yv - \yvc}\\
        &\leq \sqrt{2\lossyty} + \norm{\veps}\\
        &\leq \sqrt{2\xi\lossyzy}\exp\pa{-\frac{\sigminjgz\sigminA}{4}t} + \norm{\veps}.
    \end{align*}

    Thus, choosing $t\geq\frac{4}{\sigminjgz\sigminA}\log\pa{\frac{\sqrt{2\xi\lossyzy}}{\norm{\veps}}}$ gives~\eqref{eq:yrate}.
    
    \item We recall that by Lemma~\ref{lemma:link_params_singvals}, $\thetavt\in \Ball_{R'}\pa{\thetavz}$ for all $t \geq 0$, which in turn entails that $\xvt\in \Sigma'$ for all $t\geq0$. Then, we have using~\ref{ass:A_inj} the following chain of inequalities:
    \begin{align*}
        \norm{\xvt - \xvc} &\leq \norm{\xvt - \xvcsigma} + \dist\pa{\xvc,\Sigma'}\\
        &\leq \lmin(\fop;T_{\Sigma'}(\xvcsigma))\inv\pa{\norm{\yvt - \fop\xvcsigma}} + \dist\pa{\xvc,\Sigma'}\\
        &\leq \lmin(\fop;T_{\Sigma'}(\xvcsigma))\inv\pa{\norm{\yvt - \yv} + \norm{\yv - \fop\xvc} + \norm{\fop(\xvc - \xvcsigma)}} \\
        &\quad + \dist\pa{\xvc,\Sigma'}\\
        &\leq\frac{\sqrt{2\xi\lossyzy}\exp\pa{-\frac{\sigminjgz\sigminA}{4}t}}{\lmin(\fop;T_{\Sigma'}(\xvcsigma))} + \frac{\norm{\veps}}{\lmin(\fop;T_{\Sigma'}(\xvcsigma))} \\
        &+ \pa{1 + \frac{\norm{\fop}}{\lmin(\fop;T_{\Sigma'}(\xvcsigma))}}\dist(\xvc,\Sigma'),
    \end{align*}
    which proves~\eqref{eq:xrate}.
    \end{enumerate}
\end{proof}

\subsubsection{Proof of Theorem~\ref{th:dip_two_layers_converge}}\label{subsec:dip_two_layers_converge}

Our proof is in the same vein as that of~\cite[Theorem~4.1]{buskulic2024convergenceJournal}. However, we will improve not only the scaling but we will also accommodate better the linear operator, the new form of $R'$ and the presence of $\eta$ within it, since the latter depends on $\sigminjgz$. 

We start by providing a bound on the Lipschitz constant of $\jcal$ which is slightly tighter than the one in  \cite{buskulic2024convergenceJournal}.
\begin{lemma}\label{lemma:lip-Jacobian-both-layers}
Suppose that assumptions \ref{ass:phi_first_diff}, \ref{ass:u_sphere} and \ref{ass:v_init} are satisfied. For the one-hidden layer network \eqref{eq:dipntk}, we have for any $\thetavz$ and $\rho > 0$:
\[
\Lip_{\Ball(\thetavz,\rho)}(\jcal) \leq 2B (1+nD+\rho))\sqrt{\frac{1}{k}} .
\]
\end{lemma}

\begin{proof}
Let $\thetav \in \R^{k(d + n)}$ be the vectorized form of any parameters $(\Wv,\Vv)$ of the network . The Jacobian $\Jg$ at $\thetav$ reads 
\begin{equation}\label{eq:jacobian}
\jtheta = \frac{1}{\sqrt{k}}
\begin{bmatrix}
\phi(\Wv^1\uv) \Id_n & \ldots & \phi(\Wv^k\uv) \Id_n & \phi'(\Wv^1\uv)\Vv_1\uv\tp & \ldots & \phi'(\Wv^k\uv)\Vv_k\uv\tp
\end{bmatrix} .
\end{equation}
It then follows that $\forall \thetav, \thetavalt \in \Ball(\thetavz,\rho)$, 
\begin{equation}\label{eq:lipjacobian}
\begin{aligned}
&\norm{\jtheta - \jcal(\thetavalt)}^2=\norm{\pa{\jtheta - \jcal(\thetavalt)}\pa{\jtheta - \jcal(\thetavalt)}\tp} \\
&=\frac{1}{k} \norm{\sum_{i=1}^k \pa{\pa{\phi(\Wv^i\uv)-\phi(\Wvalt^i\uv)}^2 \Id_n + \pa{\phi'(\Wv^i\uv)\Vv_i - \phi'(\Wvalt^i\uv)\Vvalt_i}\pa{\phi'(\Wv^i\uv)\Vv_i - \phi'(\Wvalt^i\uv)\Vvalt_i}\tp}} \\
&\leq\frac{1}{k} \sum_{i=1}^k\pa{\pa{\phi(\Wv^i\uv)-\phi(\Wvalt^i\uv)}^2 + \norm{\phi'(\Wv^i\uv)\Vv_i - \phi'(\Wvalt^i\uv)\Vvalt_i}^2} \\
&\leq\frac{1}{k} \sum_{i=1}^k\pa{\pa{\phi(\Wv^i\uv)-\phi(\Wvalt^i\uv)}^2 + 2\phi'(\Wv^i\uv)^2\norm{\Vv_i - \Vvalt_i}^2 + 2\pa{\phi'(\Wv^i\uv) - \phi'(\Wvalt^i\uv)}^2\norm{\Vvalt_i}^2}\\
&\leq\frac{1}{k} \sum_{i=1}^k\pa{B^2\norm{\Wv^i-\Wvalt^i}^2 + 2B^2\norm{\Vv_i - \Vvalt_i}^2 + 2B^2\norm{\Wv^i - \Wvalt^i}^2\norm{\Vvalt_i}^2}\\
&\leq\frac{2B^2}{k}\norm{\thetav - \thetavalt}^2 + \frac{2B^2}{k}\pa{\max_{i \in [k]} \norm{\Vvalt_i}^2}\normf{\Wv-\Wvalt}^2 .
\end{aligned}
\end{equation}
Now, for any $i \in [k]$, the following holds
\[
\norm{\Vvalt_i}^2 \leq 2\norm{\Vv_i(0)}^2 + 2 \norm{\Vvalt_i-\Vv_i(0)}^2 \leq 2\norm{\Vv_i(0)}^2 + 2 \norm{\thetav-\thetavz}^2 \leq 2nD^2 + 2\rho^2 .
\]
Plugging this into \eqref{eq:lipjacobian}  and taking the square-root, we conclude.
\end{proof}

We will also need to bound $\eta$ and $\xi$, and control the spectrum of the Jacobian $\Jg$ at the initial point $\thetavz$.
\begin{lemma}
\label{lemma:singvalues_init_both_layer}
Consider the one-hidden layer network \eqref{eq:dipntk} such that \ref{ass:phi_first_diff} holds and the initialization $\thetavz$ obeys \ref{ass:u_sphere}-\ref{ass:v_init}. We have
\begin{align*}
\norm{\Jgz} \leq \Cphi + B + C\sqrt{\frac{n}{k}}
\end{align*}
with probability at least $1-3e^{-n}$, and 
\[
\sigminjgz \geq \sqrt{\Cphi^2 + \Cphid^2}/2
\]
with probability at least $1-2n^{-1}$ provided that $k/\log(k) \geq C' n\log(n)$. Here, $C$ and $C' > 0$ are large enough absolute constants that depend on $B$, $\Cphi$, $\Cphid$ and $D$.
\end{lemma}
\begin{proof}
The second bound comes from \cite[Lemma~4.9]{buskulic2024convergenceJournal}. Let us now focus on the first one.

Arguing as in \eqref{eq:lipjacobian}, we have
\begin{align}\label{eq:bound_Jg_init}
\norm{\Jgz}^2 
&\leq\frac{1}{k} \norm{\phi(\Wv(0)\uv)}^2 + \frac{1}{k}\norm{\sum_{i=1}^k\phi'(\Wv(0)^i\uv)^2\Vv(0)_i\Vv(0)_i\tp} .
\end{align}
We first concentrate $\norm{\phi(\Wv(0)\uv)}$ around its expectation. Using \ref{ass:u_sphere}, \ref{ass:w_init} and orthogonal invariance of the Gaussian distribution, we have $\Wv(0)\uv$ is $\mathcal{N}(0,\Id_k)$. Therefore,
\begin{align*}
\Expect{}{\frac{1}{\sqrt{k}} \norm{\phi(\Wv(0)\uv)}} 
&\leq \frac{1}{\sqrt{k}} \sqrt{\sum_{i=1}^k\Expect{}{\phi(\Wv(0)^i\uv)^2}} \leq \Cphi.
\end{align*}
We also know from~\ref{ass:phi_first_diff} that $\norm{\phi(\cdot)}$ is $B$-Lipschitz. Thus by the Gaussian concentration inequality,
\begin{align*}
    &\prob{\frac{1}{\sqrt{k}}\norm{\phi(\Wv(0)\uv)} \geq \Cphi+\tau} \\
    &\leq \prob{\frac{1}{\sqrt{k}} \norm{\phi(\Wv(0)\uv)} \geq \Expect{}{\frac{1}{\sqrt{k}} \norm{\phi(\Wv(0)\uv)}} + \tau} \leq \exp\pa{-\frac{\tau^2k}{2B^2}}.
\end{align*}
By choosing $\tau=B\sqrt{\frac{2n}{k}}$, we obtain that
\begin{equation}\label{eq:bndphiWu}
\frac{1}{\sqrt{k}} \norm{\phi(\Wv(0)\uv)} \leq \Cphi + B\sqrt{\frac{2n}{k}}
\end{equation}
with probability at least $1 - e^{-n}$.

We now turn to bounding the second term of~\eqref{eq:bound_Jg_init}. We first note that by \ref{ass:phi_first_diff}, we have 
\[
\frac{1}{k}\norm{\sum_{i=1}^k\phi'(\Wv(0)^i\uv)^2\Vv(0)_i\Vv(0)_i\tp} \leq \frac{B^2}{k} \norm{\Vv(0)\tp}^2 . 
\]
By~\ref{ass:v_init} and \cite[Example~5.8]{vershynin_introduction_2012}, the entries of $\Vv(0)$ are centered sub-gaussian random variables. Since they are also independent, we get from~\cite[Lemma~5.24]{vershynin_introduction_2012} that the columns $\Vv(0)_i$ are independent centered sub-gaussian random vectors, with sub-gaussian norm $K \eqdef CD$, where $C$ is an absolute constant. They are also isotropic thanks to~\ref{ass:v_init}. We are then in position to invoke~\cite[Theorem~5.41]{vershynin_introduction_2012} to assert that 
\[
\prob{\frac{1}{\sqrt{k}}\norm{\Vv(0)\tp} \geq 1 + (c_K^{-1/2}+C_K)\sqrt{\frac{n}{k}}} \leq 2e^{-n} ,
\]
where $c_K, C_K > 0$ are absolute constants that depend only on the sub-gaussian norm $K$ (hence on $D$). Plugging the last bounds into~\eqref{eq:bound_Jg_init} and then into \eqref{eq:bound_gradJg_init}, and using a union bound, we get the claim.
\end{proof}

We finally need to provide a bound on the initial loss $\lossyzy$, similarly to \cite[Lemma~4.11]{buskulic2024convergenceJournal}. Although the conclusion there was true, the proof used an independence argument to bound $\norm{\xvz}$ which was incorrect. Here, we will fix this using Hoeffding's inequality for sub-gaussian  variables.  
\begin{lemma}\label{lemma:bound_initial_misfit}
Suppose that \ref{ass:phi_first_diff} holds and the initialization $\thetavz$ obeys \ref{ass:u_sphere}-\ref{ass:v_init}. Then
\begin{align*}
\norm{\yv(0) - \yv} \leq C\norm{\fop}\pa{\Cphi + B\sqrt{\frac{2n}{k}}}\sqrt{n} + \norminf{\fop\xvc}\pa{1+ \SNR^{-1}}\sqrt{m} ,
\end{align*}
with probability at least $1 - 2e^{-(n-1)}$, where $C > 0$ is an absolute constant that depends on $D$.
\end{lemma}

\begin{proof}
We have
\[
\norm{\yvz - \yv} \leq \norm{\fop}\norm{\xvz} + \sqrt{m} \norminf{\fop\xvc}\pa{1+ \SNR^{-1}} ,
\]
where $\xvz = \gdipz = \frac{1}{\sqrt{k}} \sum_{i=1}^k \phi(\Wv^i(0)\uv)\Vv_i(0)$. Denote $\av \eqdef \frac{1}{\sqrt{k}} \phi(\Wv(0)\uv)$. We are going to use a covering argument to bound $\norm{\xvz}$. Let $\Net_\epsilon$ be an $\epsilon$-net of $\sph^{n-1}$ for some $\epsilon \in ]0,1[$. Let $\sv \in \sph^{n-1}$ such that $\norm{\xvz} = \dotprod{\xvz}{\sv}$. Let $\zv \in \Net_{\epsilon}$ which approximates $\sv$ as $\norm{\sv - \zv} \leq \epsilon$. We have
\[
\abs{\abs{\dotprod{\xvz}{\sv}} - \abs{\dotprod{\xvz}{\zv}}} \leq \epsilon\norm{\xvz} .
\]
Thus
\[
|\dotprod{\xvz}{\zv}| \geq |\dotprod{\xvz}{\sv}| - \epsilon\norm{\xvz} = (1-\epsilon)\norm{\xvz} .
\]
This implies that
\[
\norm{\xvz} \leq (1-\epsilon)^{-1}\sup_{\zv \in \Net_{\epsilon}} |\dotprod{\xvz}{\zv}|  = (1-\epsilon)^{-1}\sup_{\zv \in \Net_{\epsilon}} \abs{\sum_{i=1}^k \av_i \dotprod{\Vv_i(0)}{\zv}} .
\]
We then have
\begin{align*}
\prob{\norm{\xvz} \geq \delta} \leq \prob{\sup_{\zv \in \Net_{\epsilon}} \abs{\sum_{i=1}^k \av_i \dotprod{\Vv_i(0)}{\zv}} \geq (1-\epsilon)\delta \; \bigg| \norm{\av} < \nu} + \prob{\norm{\av} \geq \nu} .
\end{align*}
Let us fix $\zv \in \sph^{n-1}$. By assumption \ref{ass:v_init} and \cite[Lemma~5.9]{vershynin_introduction_2012}, $\dotprod{\Vv_i(0)}{\zv}$ are independent zero-mean sub-gaussian random variables with sub-gaussian norm $K = C'D$, where $C'$ is an absolute constant. It then follows from Hoeffding's inequality (\cite[Proposition~5.10]{vershynin_introduction_2012}) that
\[
\prob{\abs{\sum_{i=1}^k \av_i \dotprod{\Vv_i(0)}{\zv}} \geq (1-\epsilon)\delta \; \bigg| \norm{\av} < \nu} \leq e . e^{-\frac{c(1-\epsilon)^2\delta^2}{K^2\nu^2}} ,
\]
where $c > 0$ is an absolute constant. A union bound then yields
\[
\prob{\sup_{\zv \in \Net_{\epsilon}} \abs{\sum_{i=1}^k \av_i \dotprod{\Vv_i(0)}{\zv}} \geq (1-\epsilon)\delta \; \bigg| \norm{\av} < \nu} \leq e |\Net_{\epsilon}|e^{-\frac{c(1-\epsilon)^2\delta^2}{K^2\nu^2}} .
\]
Taking $\epsilon=1/2$, we have $|\Net_{\epsilon}| \leq 5^n$; see \cite[Lemma~5.2]{vershynin_introduction_2012}. Moreover, we know from \eqref{eq:bndphiWu} that
\[
\prob{\norm{\av} \geq \Cphi + B\sqrt{\frac{2n}{k}}} \leq e^{-n} .
\]
Taking $\nu=\Cphi + B\sqrt{\frac{2n}{k}}$ and $\delta=2K\nu\sqrt{\frac{3n}{c}}$, we get the claim.
\end{proof}

\begin{proof}[Proof of Theorem~\ref{th:dip_two_layers_converge}] 
The goal is to show that~\eqref{eq:bndR} holds with high probability under the given scaling. We start by upper-bounding $R'$. We can invoke Lemma~\ref{lemma:singvalues_init_both_layer} to infer that, whenever $k \gtrsim n\log(n)\log(k)$, with probability at least $1 - 3e^{-n} - 2n^{-1}$,
\[
\eta \lesssim \frac{\max\pa{\sigminA,(c_1)^\frac{1}{4}}}{\min\pa{{\sigminA^2}, (c_2)^\frac{1}{4}}} \qandq \xi \lesssim 1+\kappa(\fop)^2 .
\]
Using Lemma~\ref{lemma:lip-Jacobian-both-layers} and Lemma~\ref{lemma:singvalues_init_both_layer}, and arguing similarly to the first part of the proof of~\cite[Theorem~4.1]{buskulic2024convergenceJournal}, we have
\begin{align}\label{eq:R_bound}
    R \gtrsim \pa{\frac{k}{n}}^{1/4}
\end{align}
with the same probability as above. Now, Lemma~\ref{lemma:bound_initial_misfit} allows to assert that
\[
\sqrt{\lossyzy} \lesssim \norm{\fop}\sqrt{n} + \norminf{\fop\xvc}\pa{1+ \SNR^{-1}}\sqrt{m} 
\]
with probability at least $1 - 2e^{-(n-1)}$. Piecing all these bounds together $R'$, and using a union bound, one sees that
\begin{align}\label{eq:R'_bound}
R' \lesssim \frac{\max\pa{\sigminA,(c_1)^\frac{1}{4}}}{\min\pa{{\sigminA^2}, (c_2)^\frac{1}{4}}}(1+\kappa(\fop))\pa{\norm{\fop}\sqrt{n} + \norminf{\fop\xvc}\pa{1+ \SNR^{-1}}\sqrt{m}}
\end{align}
with probability at least $1 - 5e^{-(n-1)}-2n^{-1}$. Combining \eqref{eq:R_bound} and \eqref{eq:R'_bound} and using that $\pa{a+b}^4\leq 8\pa{a^4 + b^4}$ for $a,b\in\R$, we get the claim.
\end{proof}

\section{Discrete Setting}\label{sec:discrete}

Let us now turn to the discretization of~\eqref{eq:DIN} using explicit finite differences approximation. This gives a first-order (i.e., gradient-based) scheme summarized in Algorithm~\ref{alg:traingdinertiallocal}.

\begin{algorithm2e}[h]
    \SetAlgoLined
    \KwIn{$\thetav_{-1} = \thetavz$; $s_0 > 0$; $\delta \in ]0,2[$; $\rho \in ]0,1[$; $\alpha > 0$; $\beta > 0$.}
    \For{$\tau=0,1,\ldots$}{
    Compute
    \begin{equation}
    \begin{aligned}
	\qvtau &= \thetavtau + \alpha s_\tau\pa{\thetavtau - \thetavtaumoins} - \beta s_\tau^2\pa{\gradlossthetatau - \gradlossthetataumoins},\\
	\thetavtauplus &= \qvtau - s_\tau\gradlossthetatau
	\end{aligned}
	\label{eq:DIN_discr}
	\end{equation}
    with $s_\tau = \rho^{i_\tau}s_0$, where $i_\tau$ is the smallest nonnegative integer such that 
    \begin{gather*}
    \lossy(\fop \gv(\uv,{\thetavtauplus}))) - \lossy(\fop \gv(\uv,\thetavtau))) - \dotprod{\nabla_{\thetav} \lossy(\fop \gv(\uv,\thetavtau)))}{{\thetavtauplus}-\thetavtau} \\
    \leq \frac{\delta}{2s_{\tau}} \norm{{\thetavtauplus}-\thetavtau}^2 \\
    \text{and} \\
    \norm{\nabla_{\thetav} \lossy(\fop \gv(\uv,\thetavtauplus))) - \nabla_{\thetav} \lossy(\fop \gv(\uv,\thetavtau)))} \leq \frac{\delta}{s_\tau}\norm{{\thetavtauplus}-\thetavtau} .
    \end{gather*}
    }    
\caption{}
\label{alg:traingdinertiallocal}
\end{algorithm2e}
As in the continuous case, $\alpha$ is the "momentum" parameter which controls the friction while $\beta$ controls the geometric "Hessian"-driven\footnote{The quotation marks is because the Hessian does not appear explicitly but is rather approximated with the difference of gradients.} damping. The choice of the parameter sequences $\alpha s_\tau$ and $\beta s_\tau^2$ may seem cryptic at this stage, and is not stemming precisely from the time discretization of the continuous dynamic. We will however clarify later the reasons behind this choice which is flexible enough to get the desired convergence behaviour under solely local Lipschitz continuity of the objective gradient. Indeed, global Lipschitz continuity allows to take a standard upper-bound on the choice of the step-size $s_\tau$. However, such an assumption is unrealistic when training neural networks. To cope with this, a line search procedure with backtracking is crucial which poses additional technical difficulties that we must deal with carefully. 

\begin{remark}
It is worth mentioning at this stage that one can replace the backtracking update in our algorithm by $s_{\tau} = \rho^{i_\tau}s_{\tau-1}$. This update may have some benefits in practice. Our results and proofs extend readily to this case by a mild adaptation of Lemma~\ref{lem:finiteterm}. Therefore, we will not elaborate more on it.
\end{remark}

\subsection{Convergence result}
In the next theorem, we give sufficient conditions on $(\alpha,\beta,\delta)$ that ensure linear convergence of the network training to a zero-loss solution. We also provide the convergence rates as well as global convergence of the whole sequence $\thetavseqtau$.
\begin{theorem}\label{thm:main_discr}
Assume that~\ref{ass:l_MSE} and~\ref{ass:phi_first_diff} hold. Let $\thetavseqtau$ be the sequence generated by~Algorithm~\ref{alg:traingdinertiallocal} with the parameters $(\alpha,\beta,\delta)$ satisfying $s_0 \geq 1$ and $0 < 2\delta_2 < s_0^{-1}(1-\delta/2)$, where $\delta_2 = \frac{\alpha + \beta \delta}{2}$. Moreover, let the initialization $\thetav_0$ be such that
    \begin{equation}\label{eq:bndRdiscr} 
    \sigminjgz > 0 \qandq R' < R
    \end{equation}
    with
    \begin{align}
    	&R'= \frac{\sqrt{2}}{\delta_1(1-2s_0\delta_2)} \pa{\frac{2}{\underline{s}\sigminjgza} + \frac{1}{\sqrt{\delta_2}s_0}} \sqrt{\lossy(\yv_0)} \qandq \\
        &R = \frac{\sigminjgz}{2\Lip_{\Ball(\thetavz,R)}(\Jg)}
    \end{align}
    where $\delta_1 = s_0^{-1}\pa{1 - \frac{\delta}{2}} - 2\delta_2$ and $0 < \underline{s} \eqdef \inf_{\tau \in \N} s_{\tau} \leq \overline{s} \eqdef \sup_{\tau \in \N} s_{\tau} \leq s_0$. Then, the loss converges linearly to $0$ with
    \begin{equation}\label{eq:lossrate}
    \begin{aligned}
   &\lossy(\yv_\tau) \leq \frac{\delta R'^2}{2\underline{s}}\pa{\frac{\rho}{1 + \rho}}^\tau\\
   &\text{where } \rho \leq 8\delta_1\inv(1 - 2s_0\delta_2)\inv\pa{\frac{1}{\underline{s}\sigminjgza} + \frac{1}{2\sqrt{\delta_2}s_0}}^2 .\\
    \end{aligned}
    \end{equation}
    In addition, $\thetavseqtau$ converges linearly to a global minimizer $\thetav_\infty$ of \eqref{eq:minP} with
    \begin{align}
        \norm{\thetav_\tau - \thetav_\infty} \leq R'\pa{\frac{\rho}{1 + \rho}}^{\tau/2} .
    \end{align}
    If, moreover, \eqref{ass:A_inj} holds, then
    \begin{align}\label{eq:xratediscrete}
    \begin{split}
    \norm{\xvtau - \xvc} \leq\frac{\sqrt{\delta/\underline{s}} R'}{\lmin(\fop;T_{\Sigma'}(\xvcsigma))}\pa{\frac{\rho}{1 + \rho}}^{\tau/2} + \frac{\norm{\veps}}{\lmin(\fop;T_{\Sigma'}(\xvcsigma))} \\
        + \pa{1 + \frac{\norm{\fop}}{\lmin(\fop;T_{\Sigma'}(\xvcsigma))}}\dist(\xvc,\Sigma').
    \end{split}
    \end{align}
\end{theorem}

This theorem ensures that the neural network can be trained to zero loss using Algorithm~\ref{alg:traingdinertiallocal} with a proper choice of $\alpha$, $\beta$ and $\delta$. The condition $s_0(\alpha + \beta\delta) < 1 - \frac{\delta}{2}$ balances the effect of viscous (momentum) and Hessian damping with respect to the user-chosen parameter $\delta$ of the backtracking procedure to ensure convergence of the network training. 


Understanding more precisely the effect of the choice of $\alpha$, $\beta$ and $\delta$, for fixed $s_0$, on the convergence guarantees in this discrete setting boils down to understanding their role in $\delta_1$ and $\delta_2$, which in turn influence both $R'$ and our convergence rates. First, we see that the closer $\delta$ is to $2$, the more limited is the choice of $\alpha$ and $\beta$ in order to comply with the condition $s_0(\alpha + \beta\delta) < 1 - \frac{\delta}{2}$. Furthermore, this means that both $\delta_1$ and $\delta_2$ would go towards $0$, hence making  $R'$ scaling as $O(\delta\inv\delta_2^{-1/2}(1-\delta/2)\inv)$ and $\rho$ as $O(\delta\inv\delta_2^{-1}(1-\delta/2)\inv)$. This regime is undesirable as it may induce a very slow training convergence rate. On the other hand, choosing $\delta$ smaller allows for larger choices of $\alpha$ and $\beta$ to balance between $\delta_1$ and $\delta_2$. Indeed, for fixed $s_0$, one can decide to keep $\delta_2$ larger at the expense of shrinking $\delta_1$ and vice versa. Observe also that choosing $\delta$ small may have a cost by potentially increasing the backtracking procedure termination iteration. This would then make $\underline{s}$ smaller hence increasing $\rho$ and $R'$. Thus, there is a clear tradeoff in the choice of $\delta$, $\alpha$ and $\beta$.

Observe that under our initialization condition, our result states that the parameters of the network remain in a ball near that initialization on which $\sigmin(\Jg(\thetavseqtau))$ is bounded away from zero, hence verifying the {\L}ojasiewicz inequality with exponent $1/2$, hence the linear convergence rate. For the ill-conditioned case, $\rho$ scales as $O\pa{\frac{1}{\delta_1{\underline{s}^2\sigminjgz^2\sigminA^2}}}$, hence giving the convergence rate $\frac{1}{1+c\delta_1\underline{s}^2{\sigminjgza}^2}$ for some constant $c > 0$. Our estimate of the convergence rate seems overly pessimistic as it strictly larger than the convergence rate $\frac{1}{1+c{\sigminjgza}}$ known to be the optimal rate of first-order methods for strongly convex $L$-smooth objectives \cite{nesterov2013introductory}. Note however that we are dealing with a nonconvex objective whose gradient is only locally Lipschitz continuous. 

Whether our estimate of the rate can be improved or not is an open question. A possible way to have a tighter estimate is to to lift the problem to the product space $\pa{\thetavtau - \thetav_{\infty}, \thetav_{\tau-1} - \thetav_{\infty}}$ and using a linearization of $\nabla_{\thetav}\lossy(\fop\gv(\uv,\thetav))$ around $\thetav_{\infty}$, and then studying the spectral properties of the resulting matrix in the linearization, see \cite{polyak_methods_1964}. We would like to explore this further in a future work.

Theorem~\ref{th:dip_two_layers_converge} can be adapted to the new form of $R'$ and $R$ in Theorem~\ref{thm:main_discr} with minor modifications. The resulting scaling of the network architecture will be similar. We refrain from giving the details which are left to the reader. 

\subsection{Proofs}

\begin{lemma}[Finite termination and well-definedness]\label{lem:finiteterm}
The backtracking procedure in Algorithm~\ref{alg:traingdinertiallocal} terminates in a finite number of iterations and $\overline{s} \eqdef \sup_{\tau \in \N} s_{\tau} \leq s_0$. If the sequence $\thetavseqtau$ is bounded, then $\underline{s} \eqdef \inf_{\tau \in \N} s_{\tau} > 0$.
\end{lemma}

\begin{proof}
To lighten notation, let $f \eqdef \lossy \circ \fop \circ \gv(\uv,\cdot)$, and denote the Bregman divergence of $f$ as
\[
D_f(\widetilde{\thetav},\thetav) \eqdef f(\widetilde{\thetav}) - f({\thetav}) - \dotprod{\nabla f(\thetav)}{\widetilde{\thetav} - {\thetav}} .
\]
We write generically each iteration of Algorithm~\ref{alg:traingdinertiallocal} as  
\begin{gather*}
\thetavplus(\mu_i) \eqdef \thetav + \alpha_i\pa{\thetav - \thetavmoins} - \beta_i\pa{\nabla f(\thetav) - \nabla f(\thetavmoins)} - \mu_i \nabla f(\thetav), \qwhereq \\
 \mu_i = \rho^i s_{0} , \alpha_i = \alpha\mu_i, \beta_i = \beta\mu_i^2, \forall i \in \N . 
\end{gather*}
Clearly, $\thetavplus(\mu_i) \to \thetav$ as $i \to \infty$. Thus $\forall \epsilon > 0$, $\exists l_{\epsilon} > 0$ such that $\thetavplus(\mu_i) \subset \Ball(\thetav,\epsilon)$, $\forall i \geq l_{\epsilon}$. It then follows from the local Lipschitz continuity of $\nabla f$ (thanks to \ref{ass:l_MSE} and~\ref{ass:phi_first_diff}) and the descent lemma \cite[Lemma~2.64(i)]{BauschkeBook} that $\exists L_{\epsilon} > 0$ such that $\forall i \geq l_{\epsilon}$, 
\begin{equation}\label{eq:desclemloc}
\begin{aligned}
\norm{\nabla f(\thetavplus(\mu_i)) - \nabla f(\thetav)} \leq L_{\epsilon} \norm{\thetavplus(\mu_i)-\thetav} \qandq
D_f(\thetavplus(\mu_i),\thetav) \leq \frac{L_{\epsilon}}{2} \norm{\thetavplus(\mu_i)-\thetav}^2 .
\end{aligned}
\end{equation}
Assume by contradiction that the backtracking procedure does not terminate. That is, for all $i \geq 0$,
\[
\mu_i \norm{\nabla f(\thetavplus(\mu_i)) - \nabla f(\thetav)} > \delta \norm{\thetavplus(\mu_i)-\thetav} \qorq 
\mu_i D_f(\thetavplus(\mu_i),\thetav) > \frac{\delta}{2}\norm{\thetavplus(\mu_i)-\thetav}^2 .
\]
This together with \eqref{eq:desclemloc} entails that for all $i \geq l_{\epsilon}$,
\begin{gather*}
\delta \norm{\thetavplus(\mu_i)-\thetav} < \mu_i \norm{\nabla f(\thetavplus(\mu_i)) - \nabla f(\thetav)} \leq \mu_i L_{\epsilon} \norm{\thetavplus(\mu_i)-\thetav} \\
\qorq \\
\frac{\delta}{2}\norm{\thetavplus(\mu_i)-\thetav}^2 < \mu_i D_f(\thetavplus(\mu_i),\thetav) \leq \frac{\mu_i L_{\epsilon}}{2} \norm{\thetavplus(\mu_i)-\thetav}^2 .
\end{gather*}
Simplifying gives in both cases that $\delta < \mu_i L_{\epsilon}$. Passing to the limit as $i \to \infty$ yields $\delta=0$, a contradiction. 

The fact that $\overline{s} \leq s_0$ is immediate. We will now show that $\underline{s} > 0$.  We have by assumption that $\thetavseqtau \subset \Omega$, for some convex bounded set $\Omega$. The descent lemma used above implies that there exists $L_\Omega \geq 0$ such that for all $\tau \geq 0$,
\[
D_f(\thetavtauplus,\thetavtau) \leq \frac{L_\Omega}{2}\norm{\thetavtauplus-\thetavtau}^2 .
\]
We now show by induction that for all $\tau \geq 0$,
\begin{equation}\label{eq:gammalb}
s_\tau \geq \min(s_0,\rho\delta L_\Omega^{-1}) > 0 . 
\end{equation}
This is obviously true for $\tau=0$. Assume that \eqref{eq:gammalb} holds at some $\tau \geq 1$. Recall that $s_{\tau+1} = \rho^{i_{\tau+1}}s_{0}$. If $i_{\tau+1} \leq i_{\tau}$ then $s_{\tau+1} \geq s_{\tau}$ and we are done. If $i_{\tau+1} \geq i_{\tau}+1$, we suppose for contradiction that $s_{\tau+1} < \min(s_0,\rho\delta L_\Omega^{-1})$. Thus, the descent property above entails that
\[
D_f(\thetavtauplus,\thetavtau) < \frac{\delta}{2\rho^{i_{\tau+1}-1}s_{0}}\norm{\thetavtauplus-\thetavtau}^2 ,
\]
meaning that the backtracking terminates at $i_{\tau+1}-1$, leading to a contradiction as it was supposed to terminate at $i_{\tau+1}$. This concludes the proof.
\end{proof}

\begin{proof}[Proof of Theorem~\ref{thm:main_discr}] 

We will first derive a Lyapunov function, then show how the parameters of the network remain bounded under~\eqref{eq:bndRdiscr} and the devised choice of $(\alpha,\beta,s_\tau))$ which gives a lower bound on $\sigmin(\Jg(\thetavtau))$ for all $\tau\in \N^*$, which finally allow us to derive convergence rates.

\smallskip

\noindent\textbf{Step~1: Lyapunov analysis.} We first perform a Lyapunov analysis by designing an appropriate energy function. Let us now observe that the update ~\eqref{eq:DIN_discr} can be equivalently written
\begin{align*}
    \thetavtauplus = \argmin_{\thetav \in \R^p} \frac{1}{2}\norm{\thetav - \qvtau + s\gradlossthetatau}^2.
\end{align*}
Using the 1-strong convexity of $\thetav \mapsto \frac{1}{2}\norm{\thetav - \qvtau + s\gradlossthetatau}^2$, we get
\begin{align}\label{eq:str_conv_discr}
    \frac{1}{2}\norm{\thetavtauplus - \qvtau + s\gradlossthetatau}^2 \leq \frac{1}{2}\norm{\thetavtau - \qvtau + s\gradlossthetatau}^2 - \frac{1}{2}\norm{\thetavtauplus - \thetavtau}^2.
\end{align}
Let us denote for short $\alpha_\tau = \alpha s_\tau$, $\beta_\tau = \beta s_\tau^{2}$, $\vtau = \thetavtau - \thetavtaumoins$ with $\vvb_0 = \mathbf{0}$ and $\zvtau =\alpha_\tau\vtau - \beta_\tau\pa{\gradlossthetatau - \gradlossthetataumoins} $. We have $\qvtau = \thetavtau + \zvtau$. Expanding the terms on both sides of~\eqref{eq:str_conv_discr}, we obtain that
\begin{align}\label{eq:dotprod_grad_vtau}
    \dotprod{\gradlossthetatau}{\vtauplus} &\leq -\frac{\norm{\vtauplus}^2}{s_\tau} + \frac{\dotprod{\zvtau}{\vtauplus}}{s_\tau}.
\end{align}
Combining \eqref{eq:dotprod_grad_vtau} with the backtracking termination condition of Algorithm~\ref{alg:traingdinertiallocal}, which is well-defined thanks to Lemma~\ref{lem:finiteterm}, we have
\begin{align*}
    &\lossytauplus \leq \lossytau + \dotprod{\gradlossthetatau}{\vtauplus} + \frac{\delta}{2s_\tau}\norm{\vtauplus}^2\\
    &\leq \lossytau + \frac{\dotprod{\zvtau}{\vtauplus}}{s_\tau} - \frac{1}{s_\tau}\pa{1 - \frac{\delta}{2}}\norm{\vtauplus}^2 \\
    &\leq \lossytau + \frac{\alpha_\tau}{s_\tau}\dotprod{\vtauplus}{\vtau} - \frac{\beta_\tau}{s_\tau} \dotprod{\vtauplus}{\gradlossthetatau - \gradlossthetataumoins} - \frac{1}{s_\tau}\pa{1 - \frac{\delta}{2}}\norm{\vtauplus}^2 \\
    &\leq \lossytau + \frac{\alpha_\tau}{s_\tau}\norm{\vtauplus}\norm{\vtau} + \frac{\delta\beta_\tau}{s_\tau^2}\norm{\vtauplus}\norm{\vtau} - \frac{1}{s_\tau}\pa{1 - \frac{\delta}{2}}\norm{\vtauplus}^2 .
\end{align*}
Applying Young's inequality twice with $\epsilon,\epsilon' >0$, and using that $s_\tau \leq s_0$, we get
\begin{align*}
\lossytauplus \leq \lossytau + \frac{\epsilon + \epsilon'}{2}\norm{\vtau}^2 - \pa{s_0^{-1}\pa{1 - \frac{\delta}{2}} - \frac{\alpha^2}{2\epsilon} - \frac{\beta^2\delta^2}{2\epsilon'}}\norm{\vtauplus}^2.
\end{align*}
Adding $\frac{\epsilon+ \epsilon'}{2}\norm{\vtauplus}^2$ on both sides gives
\begin{multline}\label{eq:descent_lyapunov_discr_1}
\lossytauplus + \frac{\epsilon + \epsilon'}{2}\norm{\vtauplus}^2 \leq \lossytau + \frac{\epsilon + \epsilon'}{2}\norm{\vtau}^2  \\
- \pa{s_0^{-1}\pa{1 - \frac{\delta}{2}} - \frac{\alpha^2}{2\epsilon} - \frac{\beta^2\delta^2}{2\epsilon'} - \frac{\epsilon+\epsilon'}{2}}\norm{\vtauplus}^2.
\end{multline}
To ensure that the last term is nonpositive, we need that
\begin{align*}
s_0^{-1}\pa{1 - \frac{\delta}{2}} - \frac{\alpha^2}{2\epsilon} - \frac{\beta^2\delta^2}{2\epsilon'} - \frac{\epsilon+\epsilon'}{2} > 0 .
\end{align*}
Optimizing over $\epsilon$ and $\epsilon'$, we obtain $\epsilon=\alpha$ and $\epsilon' = \beta \delta$. Thus, the last condition is equivalent to $s_0(\alpha+\beta\delta) < 1-\delta/2$, hence our condition imposed on the parameters. We are now in position to define our Lyapunov sequence as $\Vtau = \lossytau + \delta_2\norm{\vtau}^2$, where $\delta_2=\frac{\alpha + \beta \delta}{2}$. \eqref{eq:descent_lyapunov_discr_1} then yields
\begin{align}\label{eq:descent_lyapunov_discr}
    \Vtauplus \leq \Vtau - \delta_1 \norm{\vtauplus}^2
\end{align}
with $\delta_1 \eqdef s_0^{-1}\pa{1 - \delta/2} - 2\delta_2 > 0$. Clearly $\Vtau$ is nonnegative decreasing sequence, and thus it converges. Moreover, as $\delta_1 > 0$, we get that $\sum_{\tau \in \N}\norm{\vvb_\tau}^2 < +\infty$, entailing that $\lim_{\tau \to \infty }\norm{\vtau} = 0$. Thus the loss $\lossytau$ converges to the same limit as $\Vtau$. 

\smallskip

\noindent\textbf{Step~2: Network weights are bounded under initialization condition.} Now that we have a Lyapunov function, we would like to have a similar result as in Lemma~\ref{lemma:link_params_singvals} for the continuous case, that is, where we show that $\thetav$ will be bounded in some ball of radius $R$ given some initialization condition. These results adapted to the discrete setting are presented in the following lemma.

\begin{lemma}\label{lemma:link_params_singvals_discr}
Assume \ref{ass:l_MSE} and \ref{ass:phi_first_diff} hold. Recall $R$ and $R'$ from \eqref{eq:bndRdiscr} with $\sigminjgz > 0$. Let $\thetavseqtau$ be the sequence given by Algorithm~\ref{alg:traingdinertiallocal} and assume that $s_0(\alpha+\beta\delta) < 1-\delta/2$.
    \begin{enumerate}[label=(\roman*)]
        \item \label{claim:singvals_bounded_if_params_bounded_discr} If $\thetav \in \Ball(\thetavz,R)$ then
        \begin{align*}
            \sigmin(\jtheta) \geq \sigminjgz/2.
        \end{align*}
        
        \item  \label{claim:params_bounded_if_singvals_bounded_discr} If for all $l \in \{0,\ldots,\tau\}$, $(\thetav_l)_{l \leq \tau} \subset \Ball(\thetavz,R)$ and $\sigminjglvar \geq \frac{\sigminjgz}{2}$, then
        \begin{align*}
            \thetavtauplus \in \Ball(\thetavz,R').
        \end{align*}
        
        \item \label{claim:sigval_bounded_everywhere_discr}
        If $R' < R$, then for all $\tau \in \N$, $\thetavseqtau \subset \Ball(\thetavz,R)$ and $\sigminjgtau \geq \sigminjgz/2$.
\end{enumerate}
\end{lemma}

\begin{proof}
\begin{enumerate}[label=(\roman*)]

\item The proof of this claim is the same as that of~\cite[Lemma~3.10(i)]{buskulic2024convergenceJournal}.

\item We know that $s_l \leq s_0$ for all $l \in \N$. Moreover, since $(\thetav_l)_{l \leq \tau}$, we can invoke Lemma~\ref{lem:finiteterm} to deduce that there exists $\underline{s} > 0$ such that $s_l \geq \underline{s} > 0$ for all $l \leq \tau$. The update equation~\eqref{eq:DIN_discr} then gives
\begin{align*}
    \underline{s}\norm{\gradlossthetatau} &\leq \norm{s_\tau\gradlossthetatau} \\
    &= \norm{\thetavtauplus - \qvtau}\\
    &\leq \norm{\thetavtauplus - \thetavtau} + \norm{\thetavtau - \qvtau}\\
    &\leq \norm{\vtauplus} + (\alpha + \beta \delta)s_0\norm{\vtau} \\
    &= \norm{\vtauplus} + 2s_0\delta_2\norm{\vtau} . 
\end{align*}

Thus, we get from Step~1 that $\lim_{\tau \to \infty }\norm{\gradlossthetatau} = 0$. Let us observe that by the condition of the lemma on $\sigminjglvar$ for any $l \in \{0,\dots,\tau\}$, we have that 
\begin{align}\label{eq:Polyak_loja_mse}
    2^{-1/2}\sigminjgza\sqrt{\lossylvar} \leq  \norm{\gradlossthetalvar} \leq \frac{1}{\underline{s}}\pa{\norm{\vlvarplus} + 2s_0\delta_2\norm{\vlvar}} ,
\end{align}

Without loss of generality, we assume that $\Vlvar \neq 0$, as otherwise, the algorithm has converged and there is nothing to prove. By concavity of $\sqrt{\cdot}$, we have
\begin{align*}
    \sqrt{\Vlvarplus} - \sqrt{\Vlvar} &\leq \frac{1}{2\sqrt{\Vlvar}}\pa{\Vlvarplus - \Vlvar}\\
    {\small \eqref{eq:descent_lyapunov_discr}}&\leq \frac{-\delta_1 \norm{\vlvarplus}^2}{2\sqrt{\Vlvar}}\\
    {\small \eqref{eq:Polyak_loja_mse}}&\leq \frac{-\delta_1 \norm{\vlvarplus}^2}{2\pa{\frac{\norm{\gradlossthetalvar}}{2^{-1/2}\sigminjgza} +\sqrt{\delta_2}\norm{\vlvar}}}.
\end{align*}
Let us define $\deltaVlvar = \sqrt{\Vlvar} - \sqrt{\Vlvarplus}$. Then
\begin{align*}
    \norm{\vlvarplus}^2 &\leq \frac{2}{\delta_1} \pa{\frac{\norm{\gradlossthetalvar}}{2^{-1/2}\sigminjgza} +\sqrt{\delta_2}\norm{\vlvar}}\deltaVlvar\\
{\small \eqref{eq:Polyak_loja_mse}}    &\leq \frac{2}{\delta_1} \pa{\frac{\norm{\vlvarplus} + 2s_0\delta_2\norm{\vlvar}}{2^{-1/2}\underline{s}\sigminjgza} +\sqrt{\delta_2}\norm{\vlvar}}\deltaVlvar\\
    &\leq \frac{2\sqrt{2}}{\delta_1} \pa{\frac{\norm{\vlvarplus}}{\underline{s}\sigminjgza} +\pa{\frac{2s_0\delta_2}{\underline{s}\sigminjgza}+\sqrt{\delta_2}}\norm{\vlvar}}\deltaVlvar.
\end{align*}
By Young's inequality, we have that for any $\epsilon>0$
\begin{align*}
    \norm{\vlvarplus} \leq \frac{\deltaVlvar}{\delta_1\epsilon} + \frac{\epsilon}{\sqrt{2}}\pa{\frac{\norm{\vlvarplus}}{\underline{s}\sigminjgza} +\pa{\frac{2s_0\delta_2}{\underline{s}\sigminjgza}+\sqrt{\delta_2}}\norm{\vlvar}} .
\end{align*}
Hence
\begin{align*}
\pa{1 - \frac{\epsilon}{\sqrt{2}\underline{s}\sigminjgza}} \norm{\vlvarplus} \leq  \frac{\deltaVlvar}{\delta_1\epsilon} + \frac{\epsilon}{\sqrt{2}}\pa{\frac{2s_0\delta_2}{\underline{s}\sigminjgza}+\sqrt{\delta_2}}\norm{\vlvar} .
\end{align*}
For any  $\epsilon \in ]0,\sqrt{2}\underline{s}\sigminjgza[$, we divide by $\pa{1 - \frac{\epsilon}{\sqrt{2}\underline{s}\sigminjgza}}$ on both sides. The goal is now to choose $\epsilon$ such that 
\begin{align*}
    \frac{\epsilon\pa{2s_0\delta_2 + \sqrt{\delta_2}\underline{s}\sigminjgza}}{\sqrt{2}\underline{s}\sigminjgza - \epsilon} < 1 ,
\end{align*}
or equivalently
\begin{align}\label{eq:epsbnd}
    \epsilon < \frac{\sqrt{2}\underline{s}\sigminjgza}{1 + 2s_0\delta_2 + \sqrt{\delta_2}\underline{s}\sigminjgza} .
\end{align}
It follows that any $\epsilon$ verifying this bound entails that there is $0 < \nu < 1$ and $C_1 > 0$ such that
\begin{align}\label{eq:delta_v}
    \norm{\vlvarplus} \leq (1-\nu)\norm{\vlvar} + C_1 \deltaVlvar.
\end{align}
We then choose 
\begin{align*}
\epsilon = \frac{\sqrt{2\delta_2}s_0\underline{s}\sigminjgza}{2\sqrt{\delta_2}s_0+\underline{s}\sigminjgza/2} .
\end{align*} 
Using our condition that $2s_0\delta_2 \in ]0,1-\delta/2[$, one can easily check that such a choice obeys \eqref{eq:epsbnd}. It also gives us $\nu=1-2s_0\delta_2 \in ]\delta/2,1[$ and 
\begin{equation}
C_1 = \frac{\pa{2\sqrt{\delta_2}s_0 + \underline{s}\sigminjgza/2}^2}{\sqrt{2\delta_2}\delta_1s_0\underline{s}\sigminjgza\pa{\sqrt{\delta_2}s_0 + \underline{s}\sigminjgza/2}} . \label{eq:C_1_full}
\end{equation}
We then obtain
\begin{equation}
C_1 \leq \frac{4\pa{\sqrt{\delta_2}s_0 + \underline{s}\sigminjgza/2}}{\sqrt{2\delta_2}\delta_1s_0\underline{s}\sigminjgza} 
\leq \frac{\sqrt{2}}{\delta_1} \pa{\frac{2}{\underline{s}\sigminjgza} + \frac{1}{\sqrt{\delta_2}s_0}} . \label{eq:C_1_bound}
\end{equation}
Since~\eqref{eq:delta_v} holds for any $l\in\{0,\dots,\tau\}$ and $\thetav_{-1}=\thetavz$, we get that
\begin{align}\label{eq:sumvel}
    \norm{\thetavtauplus - \thetavz} &\leq \sum_{l=0}^{\tau+1} \norm{\vlvar} \leq \nu^{-1}\sum_{l=0}^{\tau+1} \pa{\norm{\vlvar} - \norm{\vlvarplus} + C_1 \deltaVlvar}\nonumber\\
    &\leq \nu^{-1}C_1\pa{\sqrt{V_0}} + \nu^{-1}\norm{\vvb_0} \\
    &= \nu^{-1}C_1\sqrt{\lossy(\yv_0)}\nonumber .
\end{align}

\item We prove this by induction. For $\tau=0$, the claim trivially holds. Suppose now that $R' < R$ and that $(\thetav_l)_{l \leq \tau} \subset \Ball(\thetavz,R)$ and $\sigminjglvar \geq \frac{\sigminjgz}{2}$ for all $l \in \{0,\ldots,\tau\}$. Let us now show that this also holds at $\thetavtauplus$. By the induction assumption and \ref{claim:params_bounded_if_singvals_bounded_discr}, we have $\thetavtauplus \in \Ball(\thetavz,R') \subset \Ball(\thetavz,R)$. In view of \ref{claim:singvals_bounded_if_params_bounded_discr}, we get the claim.

\end{enumerate}
\end{proof}

In view of the condition of the theorem on $(\alpha,\beta,\delta)$ and the assumption that $R'<R$, Lemma~\ref{lemma:link_params_singvals_discr}\ref{claim:sigval_bounded_everywhere_discr} applies to ensure that $\thetavseqtau \subset \Ball(\thetavz,R)$ and $\sigminjgtau \geq \frac{\sigminjgz}{2}$ for all $\tau\in\N$. Equipped with this result, we can embark from~\eqref{eq:delta_v} and pass to the limit as  $\tau \rightarrow +\infty$ to get that $\sum_{\tau \in \N}\norm{\vtau} < +\infty$, i.e., the sequence $\thetavseqtau$ has finite length, and thus it converges to some point $\thetav_{\infty}$.

%
%

\smallskip

\noindent\textbf{Step~3: Linear convergence rate.} 
Let us define $\deltatau = \sum_{l=\tau}^{+\infty} \norm{\vvb_{l+1}}$. The triangle inequality yields $\norm{\thetavtau-\thetav_{\infty}} \leq \deltatau$. Therefore it is sufficient to analyze the rate of $\deltatau$ to get that of the iterates. Summing \eqref{eq:delta_v} from $l=\tau$ to $\infty$, we have
\[
\deltataumoins = \sum_{l=\tau}^{+\infty} \norm{\vlvar} \leq \nu^{-1}\norm{\vtau} + \nu^{-1}C_1\sqrt{\Vtau} .
\]
Thus, we have the recursion
\begin{align}\label{eq:delta_tau_sum}
\deltatau 
= \deltataumoins - \norm{\vtau} \leq \frac{1-\nu}{\nu}\norm{\vtau} + \nu^{-1}C_1\sqrt{\Vtau} \leq \frac{1-\nu}{\nu}\pa{\deltataumoins - \deltatau} + \frac{C_1}{\nu}\sqrt{\Vtau} .
\end{align}
From the definition of our Lyapunov function and~\eqref{eq:Polyak_loja_mse} we get
\begin{align*}
    &\sqrt{\Vtau} \leq \frac{\norm{\vtauplus} + 2\delta_2s_0\norm{\vtau}}{2^{-1/2}\underline{s}\sigminjgza} +\sqrt{\delta_2}\norm{\vtau}\\
    &\leq \pa{\frac{\sqrt{2}}{\underline{s}\sigminjgza} + \sqrt{\delta_2}}\pa{\norm{\vtauplus} + \norm{\vtau}} \\
    &= \pa{\frac{\sqrt{2}}{\underline{s}\sigminjgza} + \sqrt{\delta_2}}\pa{\deltataumoins-\deltatauplus},
\end{align*}
where we used that $2\delta_2s_0 < 1-\delta/2 < 1$ by assumption. Plugging this into~\eqref{eq:delta_tau_sum} we get
\begin{align*}
    \deltatau &\leq \frac{1-\nu}{\nu}\pa{\deltataumoins - \deltatau} + \frac{\pa{\frac{\sqrt{2}}{\underline{s}\sigminjgza} + \sqrt{\delta_2}}C_1}{\nu}\pa{\deltataumoins-\deltatauplus} .
\end{align*}
Let us denote $C_2 = \max\pa{\pa{\frac{\sqrt{2}}{\underline{s}\sigminjgza} + \sqrt{\delta_2}}C_1,1-\nu}$. One can see, using~\eqref{eq:C_1_bound}, that $2s_0\delta_2 < 1$ and $1-\nu = 2s_0\delta_2$, that
\begin{align*}
C_2 
&\leq \max\pa{4\delta_1\inv\pa{\frac{1}{\underline{s}\sigminjgza} + \sqrt{\delta_2}}\pa{\frac{1}{\underline{s}\sigminjgza} + \frac{1}{2\sqrt{\delta_2}s_0}},2s_0\delta_2} \\
&\leq \max\pa{4\delta_1\inv\pa{\frac{1}{\underline{s}\sigminjgza} + \frac{1}{2\sqrt{\delta_2}s_0}}^2,2s_0\delta_2} .
\end{align*}
We claim that the maximum in the rhs is given by the first term. Indeed,
\begin{align*}
4\delta_1\inv\pa{\frac{1}{\underline{s}\sigminjgza} + \frac{1}{2\sqrt{\delta_2}s_0}}^2 
&=\frac{1}{\delta_1\delta_2s_0^2}\pa{\frac{2\sqrt{\delta_2}s_0}{\underline{s}\sigminjgza} + 1}^2 \geq 2 ,
\end{align*}
since by assumption $2\delta_2s_0 < 1$ which also entails $s_0\delta_1 < 1$. Therefore
\[
C_2 \leq 4\delta_1\inv\pa{\frac{1}{\underline{s}\sigminjgza} + \frac{1}{2\sqrt{\delta_2}s_0}}^2 .
\]
Using now that $\deltatauplus \leq \deltatau$, we obtain
\begin{align*}
    \deltatauplus \leq \deltatau \leq \frac{C_2}{\nu}\pa{\deltataumoins - \deltatauplus + \deltataumoins - \deltatau}
    \leq \frac{2C_2}{\nu}\pa{\deltataumoins - \deltatauplus}.
\end{align*}
Equivalently,
\[
\deltatauplus \leq \frac{\rho}{1+\rho}\deltataumoins
\]
where $\rho = \frac{2C_2}{\nu}$. This implies
\begin{align*}
\deltatau \leq \pa{\frac{\rho}{1 + \rho}}^{\tau/2} \Delta_0. 
\end{align*}
Observing that $\Delta_0 \leq \nu^{-1}C_1\sqrt{\lossy(\yv_0)}=R'$ (see \eqref{eq:sumvel}), it follows that
\begin{align*}
\norm{\thetavtau - \thetav_{\infty}} \leq \deltatau \leq R'\pa{\frac{\rho}{1 + \rho}}^{\tau/2} .
\end{align*}
By the well-posedness of the backtracking procedure, we obtain
\begin{align*}
    \lossy(\yv_\tau) \leq \frac{\delta}{2\underline{s}}\norm{\thetavtau - \thetav_{\infty}}^2 \leq \frac{\delta R'^2}{2\underline{s}}\pa{\frac{\rho}{1 + \rho}}^{\tau}.
\end{align*}
which concludes the proof.

\end{proof}
\section{Numerical Experiments}\label{sec:numerical}

In order to validate our results, we performed different numerical experiments. Throughout these experiments, we used a two-layer neural network equipped with the sigmoid activation function and where the entries of $\Wv(0)$ are sampled from a standard normal distribution and the entries of $\Vv(0)$ from a uniform distribution between $-\sqrt{3}$ and $\sqrt{3}$. The networks then obviously obey our assumptions. We train both layers of the networks.

In our first experiment shown in Figure~\ref{fig:grid_convergence}, our goal is to see the impact of the parameters $\alpha$ and $\beta$ on the convergence speed of the network applied to a simple inverse problem of relatively low dimension with $n=10$ and $m=5$ and without noise. We entries of $\xvc$ are iid samples from $\mathcal{N}(0,1)$ and those of $\fop$ from $\mathcal{N}(0,1/\sqrt{n})$. Standard random matrix theory results ensure that the non-zero singular values of $\fop$ are concentrated around $1$. To solve this problem, we use networks where $k=10^4$ and $d=1$. We trained the network for each instance using our inertial algorithm with $s_0=0.1$ fixed, and varied $\alpha$ and $\beta$ to assess their influence. We generate $50$ different problems instances characterized by an operator $\fop$, a signal $\xvc$ and a network initialization. For each pair of parameters ($\alpha,\beta)$, we computed the average number of iterations over the 50 problem instances that were necessary to achieve machine precision accuracy (i.e., $\lossy(\yvtau) \leq 10^{-14}$).

\begin{figure}[thb!]
    \centering
    \begin{subfigure}{0.49\textwidth}
        \includegraphics[width=\linewidth]{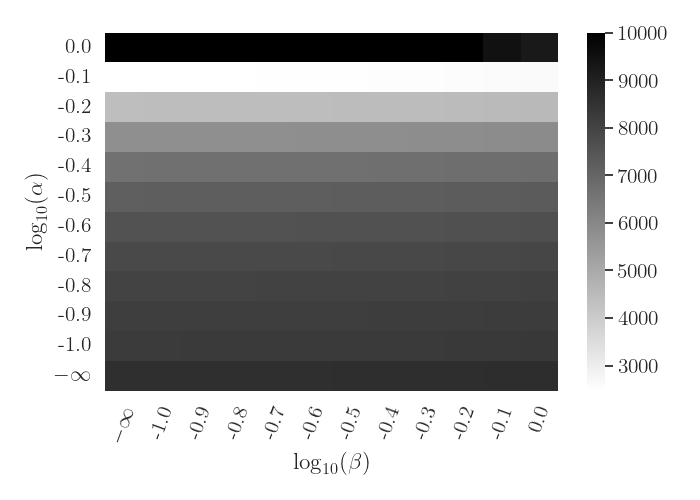}
        \caption{Number of iterations necessary for a network to converge on average, for different pairs $(\alpha,\beta)$}
    \end{subfigure}
     \begin{subfigure}{0.49\textwidth}
        \includegraphics[width=\linewidth]{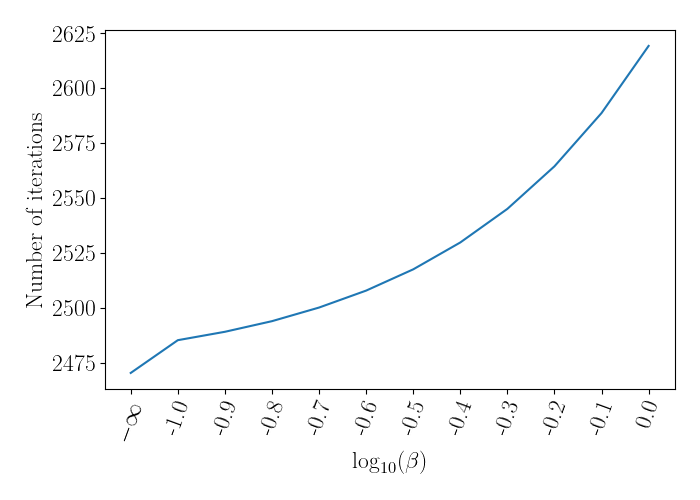}
        \caption{Effect of $\beta$ on the number of iterations necessary to converge for $\alpha=10^{-0.2}$}
        \label{fig:grid_convergence_beta}
    \end{subfigure}
    \caption{Convergence rates of an inertial system for different $\alpha$ and $\beta$. Better $\alpha$ allows for much faster convergence while for this problem, $\beta$ does not appear to be necessary.}
    \label{fig:grid_convergence}
\end{figure}

For this problem, it is obvious that $\alpha$ is the driving factor of convergence speed. We see a clear acceleration of the training of the network as $\alpha$ progresses until the network start diverging when $\alpha=1$. The acceleration we observe is very important as we go on average from 9000 iterations to converge when $\alpha=0$ (the gradient descent case typically) to only 3000 when we chose $\alpha=10^{-0.1}$. We see in Figure~\ref{fig:grid_convergence_beta} the effect of $\beta$ on the convergence when $\alpha=10^{-2}$. We observe a slight degradation of the convergence rate when $\beta$ progresses, revealing that for this problem, Hessian damping is not necessary and only the effect of viscous damping drives the acceleration. This is due to the fact that for this problem we do not observe oscillations around the minima, which is what Hessian damping helps to prevent.

In the next experiment, we kept the same setting but we fixed $\beta=0.05$ and varied both $k$ and $\alpha$. We train once again 50 networks for each pair $(k,\alpha)$ and we plot in Figure~\ref{fig:grid_convergence_k_alpha} the probability of each network to achieve machine precision loss in less than 15000 iterations. From these results we see different regimes. When the network is too small, whatever $\alpha$ is chosen, the network will not converge to a zero-loss solution, which is in agreement with our theoretical predictions. On the other side, when $\alpha$ is too large, the algorithm will not converge either whatever the size of the network. However, when the network is big enough, the choice of $\alpha$ has a clear impact on the convergence speed. It is to be noted that in some cases, much more iterations would lead to convergence, but it seems that even by taking this into account, using the right $\alpha$ does help a network to find a zero-loss solution even when its size is relatively small. This might be related to a better trap-avoidance property of inertial dynamics that would be worth investigating precisely in the future.

\begin{figure}[thb!]
    \centering
    \includegraphics[width=0.7\linewidth]{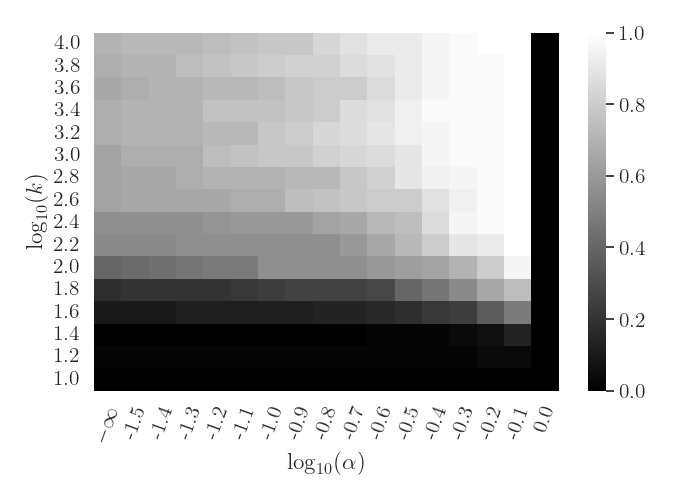}
    \caption{Empirical probability of a network to be trained in less than 15000 iterations for different $k$ and $\alpha$. Choosing $\alpha$ correctly helps when $k$ is above a certain threshold.}
    \label{fig:grid_convergence_k_alpha}
\end{figure}

For our next experiment, we explore the effects of $\alpha$ and $\beta$ on an imaging inverse problem. We consider the image as a vector in $[0,255]^{4096}$ and use a network with $k=7000$ hidden neurons, which is enough to achieve convergence empirically. We study a deconvolution problem where $\fop$ is a Gaussian kernel of standard deviation 1, and added an $\mathcal{N}(0,2.5^2)$ noise. We trained different networks using various $\alpha$ and $\beta$ and show in Figure~\ref{fig:plot_convergence_convolution} the evolution of the loss and of the distance to the true solution for each pair of parameters. We also show in Figure~\ref{fig:results_convolution} the final image obtained after a selected number of iterations for both gradient descent ($\alpha=0$ and $\beta=0$) and when $\alpha=1$ and $\beta=0.1$ which is one of the fastest converging combination.

\begin{figure}
\begin{minipage}{0.75\linewidth}
\begin{center}
\includegraphics[width=1\linewidth]{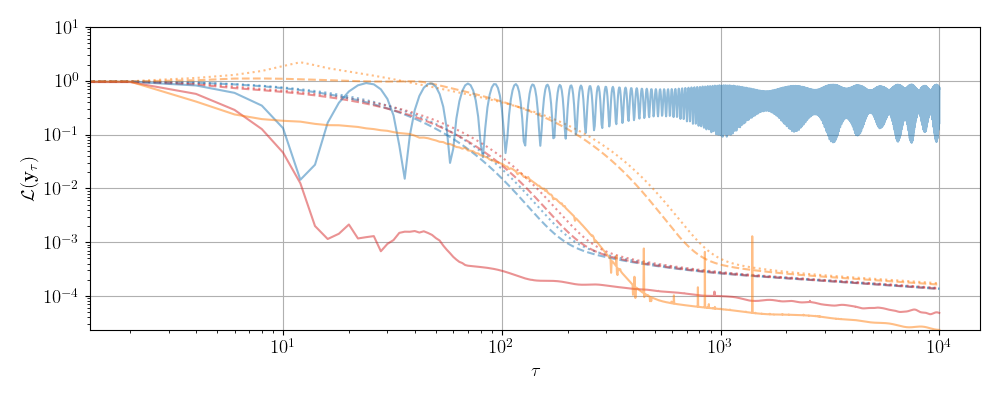}
\includegraphics[width=1\linewidth]{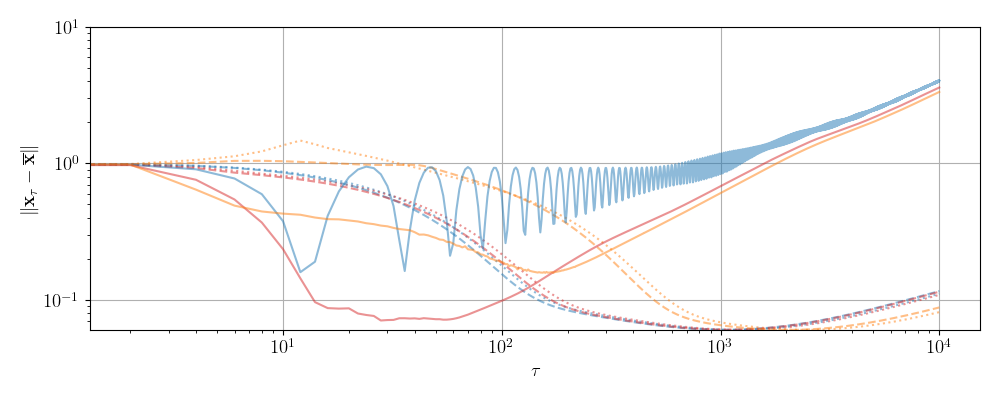}
\end{center}
\end{minipage}
\begin{minipage}{0.19\linewidth}
\includegraphics[width=1\linewidth]{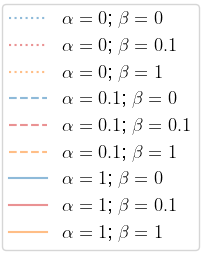}
\end{minipage}
\caption{Effects on the loss and the signal convergence of different combination of $\alpha$ and $\beta$ for a deconvolution problem. Inertial systems can provide faster convergence but they are sensible to the parameters $\alpha$ and $\beta$ and a wrong choice can prevent convergence.}
\label{fig:plot_convergence_convolution}
\end{figure}


\begin{figure}[htb!]
    \centering
    \includegraphics[width=1\linewidth]{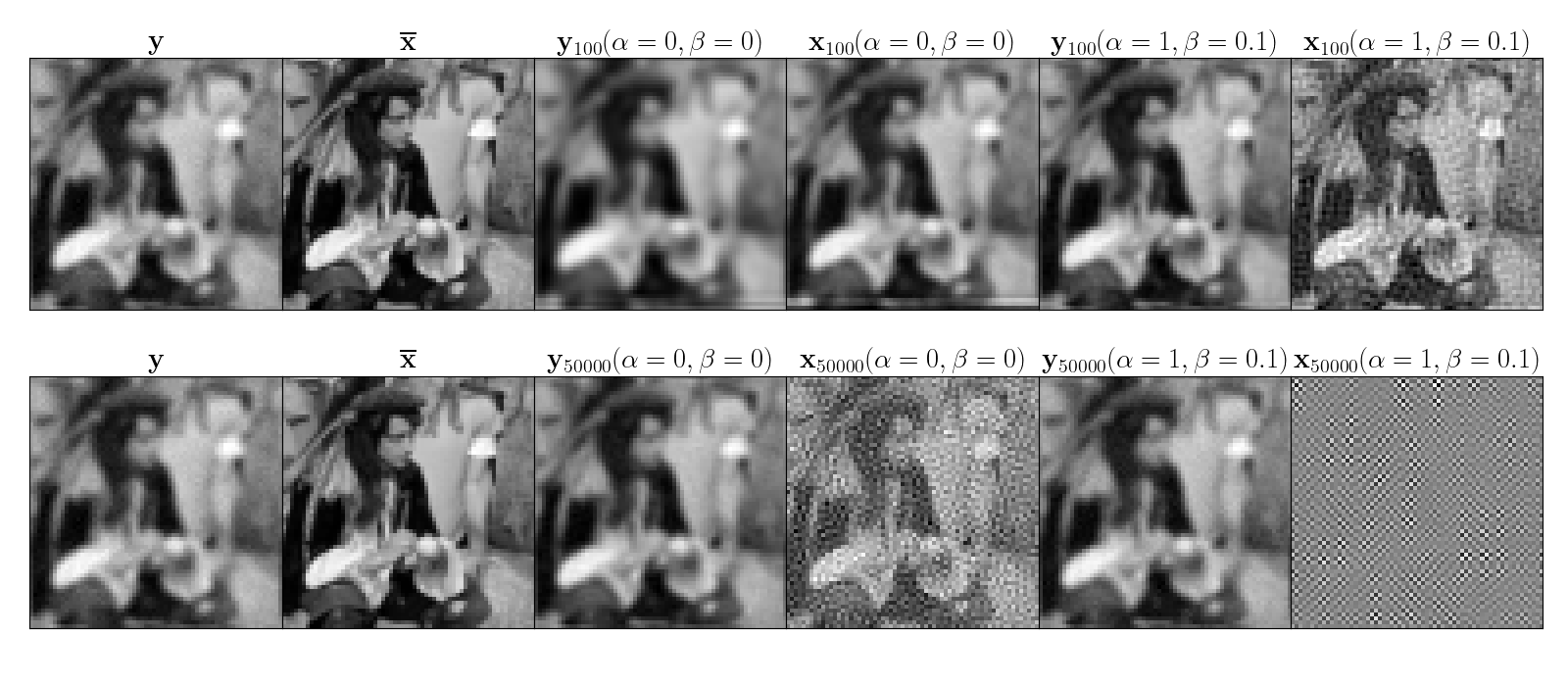}
    \caption{Qualitative results of a deconvolution problem with low noise for different $\alpha$ and $\beta$. Both gradient descent and inertial system converge in observation space but they overfit the noise in signal space even for very reduced level of noise. }
    \label{fig:results_convolution}
\end{figure}

Let us first focus on the evolution of the loss given in Figure~\ref{fig:results_convolution}. The first thing to note is that for the right combination of $\alpha$ and $\beta$ ($\alpha=1$ and $\beta\in\{0.1,1\}$), there is considerable acceleration phenomenon compared to gradient descent. More precisely, the acceleration happens at some point, either in the beginning for $(\alpha=1,\beta=0.1)$ or later on for $(\alpha=1,\beta=1)$, and then the optimization seems to continue at a similar rate as the gradient descent. Contrary to the previous toy example, this time the acceleration can only be achieved by a good combination of $\alpha$ and $\beta$ showing the interplay between viscous and Hessian dampings for more complex inverse problems. Indeed, for $(\alpha=1,\beta=0)$, we see that the loss oscillates without converging and on the other side $(\alpha=0,\beta=1)$ does converge but at a slower rate than gradient descent -- note that such phenomenon also appears for $(\alpha=0.1,\beta=1)$. Finally, choosing small $\alpha$ and $\beta$ will not provide the desired acceleration or can even hinder the convergence rate compared to gradient descent.

If we now observe the evolution of the error between the reconstructed signal and the true solution $\xvc$, we see this error decreases and then starts increasing before reaching a plateau that fits the noise level. This effect is amplified here, as the convolution operator, while being injective, is very badly conditioned ($\sigmin(\fop)\sim10^{-5}$). This validates the need for an early stopping strategy. However in our case, we have that $\norm{\veps}\sim 150$, which combined with the conditioning of the operator means that our early stopping bound is far from being reached in our experiments but we observe in Figure~\ref{fig:results_convolution} that the reconstructed image has already severely overfitted the noise after 50000 iterations. Similarly, our bound on $\norm{\xv_{\tau} - \xvc}$ is far from being reached because the noise term will be very large, showing the difficulties faced when using very badly conditioned operators. Finally, let us observe that in Figure~\ref{fig:results_convolution}, both gradient descent and the inertial system will overfit the noise, albeit at different times due to slower convergence of gradient descent.

We did a second imaging experiment where we changed the convolution operator to a better conditioned one and with higher level of noise to see how the optimization trajectories behave under these conditions. The operator $\fop$ that we are using is built from the combination of two orthonormal matrices and a diagonal matrix with entries in the range $[1,2]$ (the goal is to have the same matrices produced by an SVD). We used a Gaussian noise vector with entries iid sampled from $\mathcal{N}(0,25)$. We plot the evolution of the loss for different parameters in Figure~\ref{fig:plot_convergence_well_conditioned} and we see different behaviors than for the convolution operator. By choosing $\beta$ too large (1 here), the algorithm diverges which did not happen in the previous experiment meaning that the conditioning of the operator plays an important role in the right choice of $\alpha$ and $\beta$ to ensure convergence as we discussed after our theoretical results above. We also see that the choice $(\alpha=1,\beta=0.1)$ provides faster convergence at the beginning but then starts oscillating and reaches the convergence threshold later than gradient descent. It appears that choosing $(\alpha=0.5,\beta=0)$ provides the fastest convergence rate, indicating that maybe for this problem, Hessian damping is not as necessary and the solution landscape is smoother.

\begin{figure}
\begin{minipage}{0.75\linewidth}
\begin{center}
\includegraphics[width=1\linewidth]{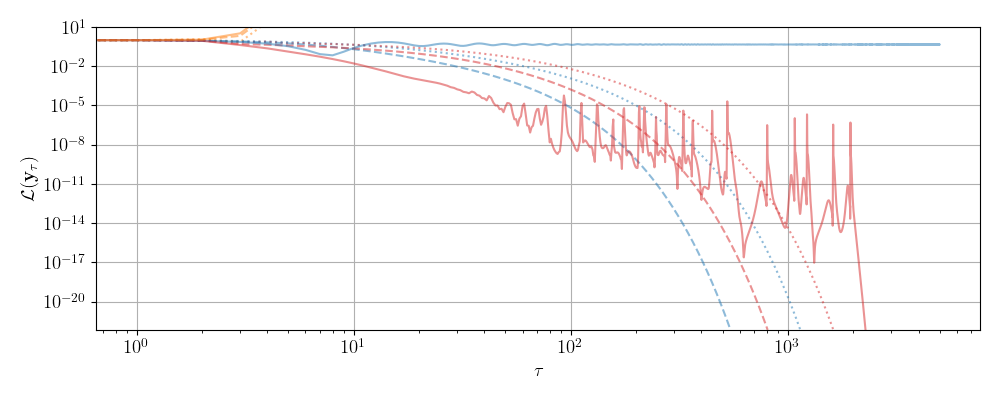}
\includegraphics[width=1\linewidth]{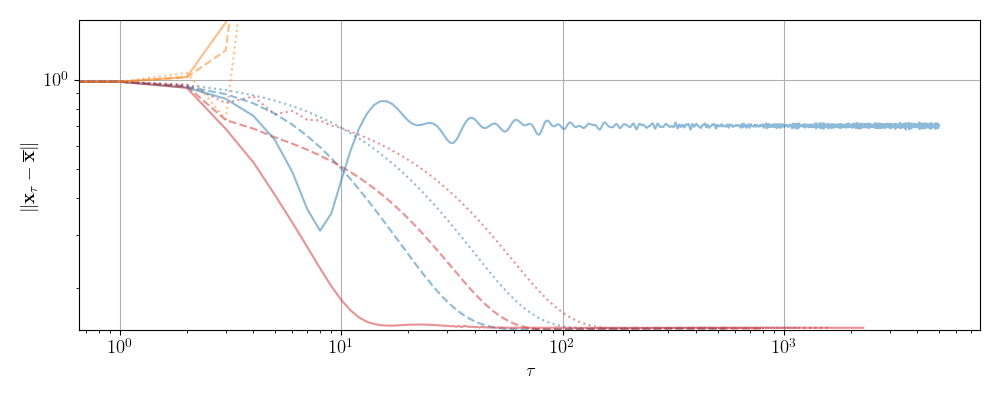}
\end{center}
\end{minipage}
\begin{minipage}{0.19\linewidth}
\includegraphics[width=1\linewidth]{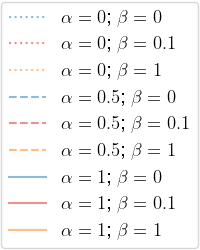}
\end{minipage}
\caption{Effects on the loss and the signal convergence of different combination of $\alpha$ and $\beta$ for a well-conditioned operator. We observe faster convergence for a variety of parameters and the networks converge to the same signal as the problem is well-posed.}
\label{fig:plot_convergence_well_conditioned}
\end{figure}

\begin{figure}[htb!]
    \centering
    \includegraphics[width=1\linewidth]{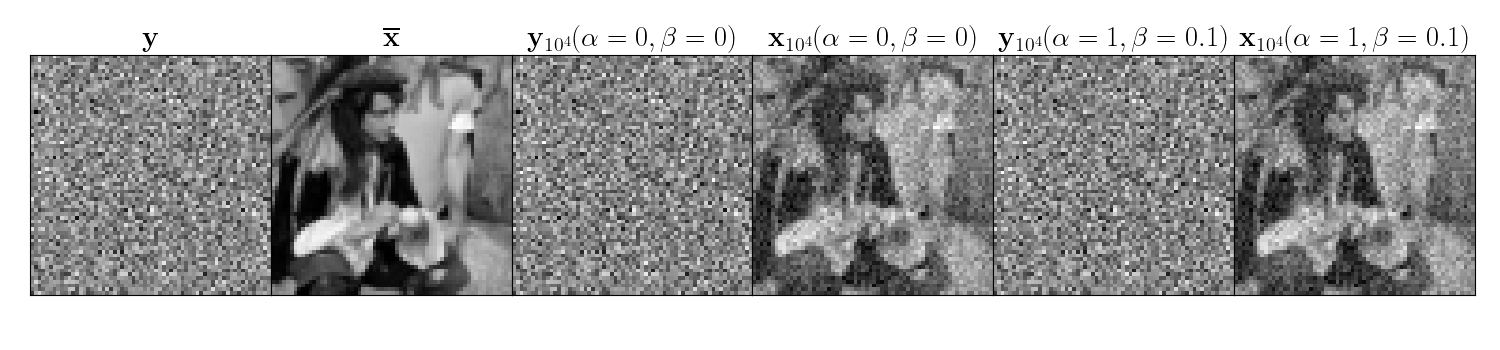}
    \caption{Qualitative results with a well-conditioned operator and heavy noise for different $\alpha$ and $\beta$. The signal is well recovered despite the heavy noise level for both gradient descent and inertial systems.}
    \label{fig:results_SVD}
\end{figure}

When we look at the evolution of the curves in Figure~\ref{fig:plot_convergence_well_conditioned}, there are some surprises. Notably the case $(\alpha=1,\beta=0.1)$ which shows these swings in the loss, is the fastest in signal space. Furthermore, we see that for a lot of cases, the algorithm converges in a stable way in the signal space and stays around the solution before overfitting the noise. This was expected as the operator is well-conditioned and thus overfitting the noise, even if it is to quite high level like here, does not require big changes in signal space. We see this effect in the qualitative results of Figure~\ref{fig:results_SVD} where the solution found by gradient descent and the inertial algorithm are very similar and do not show the same artifacts as was the case for the convolution operator.



\bibliographystyle{elsarticle-harv}
\bibliography{references}

\begin{thebibliography}{38}
\expandafter\ifx\csname natexlab\endcsname\relax\def\natexlab#1{#1}\fi
\providecommand{\url}[1]{\texttt{#1}}
\providecommand{\href}[2]{#2}
\providecommand{\path}[1]{#1}
\providecommand{\DOIprefix}{doi:}
\providecommand{\ArXivprefix}{arXiv:}
\providecommand{\URLprefix}{URL: }
\providecommand{\Pubmedprefix}{pmid:}
\providecommand{\doi}[1]{\href{http://dx.doi.org/#1}{\path{#1}}}
\providecommand{\Pubmed}[1]{\href{pmid:#1}{\path{#1}}}
\providecommand{\bibinfo}[2]{#2}
\ifx\xfnm\relax \def\xfnm[#1]{\unskip,\space#1}\fi
\bibitem[{Alvarez et~al.(2002)Alvarez, Attouch, Bolte and
  Redont}]{alvarez2002second}
\bibinfo{author}{Alvarez, F.}, \bibinfo{author}{Attouch, H.},
  \bibinfo{author}{Bolte, J.}, \bibinfo{author}{Redont, P.},
  \bibinfo{year}{2002}.
\newblock \bibinfo{title}{A second-order gradient-like dissipative dynamical
  system with hessian-driven damping.: Application to optimization and
  mechanics}.
\newblock \bibinfo{journal}{Journal de math{\'e}matiques pures et
  appliqu{\'e}es} \bibinfo{volume}{81}, \bibinfo{pages}{747--779}.
\bibitem[{Antun et~al.(2020)Antun, Renna, Poon, Adcock and
  Hansen}]{antun2020instabilities}
\bibinfo{author}{Antun, V.}, \bibinfo{author}{Renna, F.},
  \bibinfo{author}{Poon, C.}, \bibinfo{author}{Adcock, B.},
  \bibinfo{author}{Hansen, A.C.}, \bibinfo{year}{2020}.
\newblock \bibinfo{title}{On instabilities of deep learning in image
  reconstruction and the potential costs of ai}.
\newblock \bibinfo{journal}{Proceedings of the National Academy of Sciences}
  \bibinfo{volume}{117}, \bibinfo{pages}{30088--30095}.
\bibitem[{Arora et~al.(2019)Arora, Du, Hu, Li and
  Wang}]{arora_fine-grained_2019}
\bibinfo{author}{Arora, S.}, \bibinfo{author}{Du, S.}, \bibinfo{author}{Hu,
  W.}, \bibinfo{author}{Li, Z.}, \bibinfo{author}{Wang, R.},
  \bibinfo{year}{2019}.
\newblock \bibinfo{title}{Fine-grained analysis of optimization and
  generalization for overparameterized two-layer neural networks}, in:
  \bibinfo{booktitle}{International Conference on Machine Learning}, pp.
  \bibinfo{pages}{322--332}.
\bibitem[{Arridge et~al.(2019)Arridge, Maass, Ozan and
  Sch\"{o}nlieb}]{arridge_solving_2019}
\bibinfo{author}{Arridge, S.}, \bibinfo{author}{Maass, P.},
  \bibinfo{author}{Ozan, O.}, \bibinfo{author}{Sch\"{o}nlieb, C.B.},
  \bibinfo{year}{2019}.
\newblock \bibinfo{title}{Solving inverse problems using data-driven models}.
\newblock \bibinfo{journal}{Acta Numerica} \bibinfo{volume}{28},
  \bibinfo{pages}{1--174}.
\bibitem[{Attouch et~al.(2022)Attouch, Chbani, Fadili and
  Riahi}]{attouch2022first}
\bibinfo{author}{Attouch, H.}, \bibinfo{author}{Chbani, Z.},
  \bibinfo{author}{Fadili, J.}, \bibinfo{author}{Riahi, H.},
  \bibinfo{year}{2022}.
\newblock \bibinfo{title}{First-order optimization algorithms via inertial
  systems with hessian driven damping}.
\newblock \bibinfo{journal}{Mathematical Programming} , \bibinfo{pages}{1--43}.
\bibitem[{Attouch et~al.(2023)Attouch, Fadili and
  Kungurtsev}]{attouch_effect_2023}
\bibinfo{author}{Attouch, H.}, \bibinfo{author}{Fadili, J.},
  \bibinfo{author}{Kungurtsev, V.}, \bibinfo{year}{2023}.
\newblock \bibinfo{title}{On the effect of perturbations in first-order
  optimization methods with inertia and {Hessian} driven damping}.
\newblock \bibinfo{journal}{Evolution Equations and Control Theory}
  \bibinfo{volume}{12}, \bibinfo{pages}{71}.
\bibitem[{Bartlett et~al.(2021)Bartlett, Montanari and
  Rakhlin}]{bartlett_deep_2021}
\bibinfo{author}{Bartlett, P.L.}, \bibinfo{author}{Montanari, A.},
  \bibinfo{author}{Rakhlin, A.}, \bibinfo{year}{2021}.
\newblock \bibinfo{title}{Deep learning: a statistical viewpoint}.
\newblock \bibinfo{journal}{Acta numerica} \bibinfo{volume}{30},
  \bibinfo{pages}{87--201}.
\bibitem[{Bauschke and Combettes(2017)}]{BauschkeBook}
\bibinfo{author}{Bauschke, H.H.}, \bibinfo{author}{Combettes, P.L.},
  \bibinfo{year}{2017}.
\newblock \bibinfo{title}{Convex {Analysis} and {Monotone} {Operator} {Theory}
  in {Hilbert} {Spaces}}.
\newblock {CMS} {Books} in {Mathematics}, \bibinfo{publisher}{Springer
  International Publishing}, \bibinfo{address}{Cham}.
\bibitem[{Buskulic et~al.(2024a)Buskulic, Fadili and
  Qu{\'e}au}]{buskulic2024convergenceJournal}
\bibinfo{author}{Buskulic, N.}, \bibinfo{author}{Fadili, J.},
  \bibinfo{author}{Qu{\'e}au, Y.}, \bibinfo{year}{2024}a.
\newblock \bibinfo{title}{Convergence and recovery guarantees of unsupervised
  neural networks for inverse problems}.
\newblock \bibinfo{journal}{Journal of Mathematical Imaging and Vision} ,
  \bibinfo{pages}{1--22}.
\bibitem[{Buskulic et~al.(2024b)Buskulic, Fadili and
  Qu{\'e}au}]{buskulic2024descentarxiv}
\bibinfo{author}{Buskulic, N.}, \bibinfo{author}{Fadili, J.},
  \bibinfo{author}{Qu{\'e}au, Y.}, \bibinfo{year}{2024}b.
\newblock \bibinfo{title}{Recovery guarantees of unsupervised neural networks
  for inverse problems trained with gradient descent}.
\newblock \bibinfo{journal}{32nd European Signal Processing Conference} .
\bibitem[{Castera et~al.(2021)Castera, Bolte, F\'evotte and
  Pauwels}]{casterainertial}
\bibinfo{author}{Castera, C.}, \bibinfo{author}{Bolte, J.},
  \bibinfo{author}{F\'evotte, C.}, \bibinfo{author}{Pauwels, E.},
  \bibinfo{year}{2021}.
\newblock \bibinfo{title}{An inertial {N}ewton algorithm for deep learning}.
\newblock \bibinfo{journal}{J. Mach. Learn. Res.} \bibinfo{volume}{22},
  \bibinfo{pages}{1--31}.
\bibitem[{Chizat et~al.(2019)Chizat, Oyallon and Bach}]{chizat_lazy_2019}
\bibinfo{author}{Chizat, L.}, \bibinfo{author}{Oyallon, E.},
  \bibinfo{author}{Bach, F.}, \bibinfo{year}{2019}.
\newblock \bibinfo{title}{On lazy training in differentiable programming}.
\newblock \bibinfo{journal}{Advances in neural information processing systems}
  \bibinfo{volume}{32}.
\bibitem[{Du et~al.(2019)Du, Zhai, Poczos and Singh}]{du_gradient_2019}
\bibinfo{author}{Du, S.S.}, \bibinfo{author}{Zhai, X.},
  \bibinfo{author}{Poczos, B.}, \bibinfo{author}{Singh, A.},
  \bibinfo{year}{2019}.
\newblock \bibinfo{title}{Gradient {Descent} {Provably} {Optimizes}
  {Over}-parameterized {Neural} {Networks}}, in: \bibinfo{booktitle}{ICLR}, pp.
  \bibinfo{pages}{1--19}.
\bibitem[{Duff et~al.(2024)Duff, Campbell and Ehrhardt}]{duff2024regularising}
\bibinfo{author}{Duff, M.}, \bibinfo{author}{Campbell, N.},
  \bibinfo{author}{Ehrhardt, M.J.}, \bibinfo{year}{2024}.
\newblock \bibinfo{title}{Regularising inverse problems with generative machine
  learning models}.
\newblock \bibinfo{journal}{Journal of Mathematical Imaging and Vision}
  \bibinfo{volume}{66}, \bibinfo{pages}{37--56}.
\bibitem[{Fang et~al.(2021)Fang, Dong and Zhang}]{fang_mathematical_2021}
\bibinfo{author}{Fang, C.}, \bibinfo{author}{Dong, H.}, \bibinfo{author}{Zhang,
  T.}, \bibinfo{year}{2021}.
\newblock \bibinfo{title}{Mathematical models of overparameterized neural
  networks}.
\newblock \bibinfo{journal}{Proceedings of the IEEE} \bibinfo{volume}{109},
  \bibinfo{pages}{683--703}.
\bibitem[{Gottschling et~al.(2020)Gottschling, Antun, Adcock and
  Hansen}]{gottschling2020troublesome}
\bibinfo{author}{Gottschling, N.M.}, \bibinfo{author}{Antun, V.},
  \bibinfo{author}{Adcock, B.}, \bibinfo{author}{Hansen, A.C.},
  \bibinfo{year}{2020}.
\newblock \bibinfo{title}{The troublesome kernel on hallucinations: no free
  lunches and the accuracy-stability trade-off in inverse problems}.
\newblock \bibinfo{journal}{arXiv preprint arXiv:2001.01258} .
\bibitem[{Haraux(1991)}]{Haraux91}
\bibinfo{author}{Haraux, A.}, \bibinfo{year}{1991}.
\newblock \bibinfo{title}{Syst\`emes dynamiques dissipatifs et applications}.
  volume~\bibinfo{volume}{17} of \textit{\bibinfo{series}{Recherches en
  Math\'ematiques Appliqu\'ees}}.
\newblock \bibinfo{publisher}{Masson}, \bibinfo{address}{Paris}.
\bibitem[{Jacot et~al.(2018)Jacot, Gabriel and Hongler}]{jacot_neural_2018}
\bibinfo{author}{Jacot, A.}, \bibinfo{author}{Gabriel, F.},
  \bibinfo{author}{Hongler, C.}, \bibinfo{year}{2018}.
\newblock \bibinfo{title}{Neural tangent kernel: Convergence and generalization
  in neural networks}.
\newblock \bibinfo{journal}{Advances in neural information processing systems}
  \bibinfo{volume}{31}.
\bibitem[{Kaltenbacher et~al.(2008)Kaltenbacher, Neubauer and
  Scherzer}]{KaltenbacherIterative08}
\bibinfo{author}{Kaltenbacher, B.}, \bibinfo{author}{Neubauer, A.},
  \bibinfo{author}{Scherzer, O.}, \bibinfo{year}{2008}.
\newblock \bibinfo{title}{Iterative regularization methods for nonlinear
  ill-posed problems}. volume~\bibinfo{volume}{6}.
\newblock \bibinfo{publisher}{Walter de Gruyter}.
\bibitem[{Kamilov et~al.(2023)Kamilov, Bouman, Buzzard and
  Wohlberg}]{kamilov2023plug}
\bibinfo{author}{Kamilov, U.S.}, \bibinfo{author}{Bouman, C.A.},
  \bibinfo{author}{Buzzard, G.T.}, \bibinfo{author}{Wohlberg, B.},
  \bibinfo{year}{2023}.
\newblock \bibinfo{title}{Plug-and-play methods for integrating physical and
  learned models in computational imaging: Theory, algorithms, and
  applications}.
\newblock \bibinfo{journal}{IEEE Signal Processing Magazine}
  \bibinfo{volume}{40}, \bibinfo{pages}{85--97}.
\bibitem[{Kingma and Ba(2014)}]{kingma2014adam}
\bibinfo{author}{Kingma, D.P.}, \bibinfo{author}{Ba, J.}, \bibinfo{year}{2014}.
\newblock \bibinfo{title}{Adam: A method for stochastic optimization}.
\newblock \bibinfo{journal}{arXiv preprint arXiv:1412.6980} .
\bibitem[{Liu et~al.(2019)Liu, Sun, Xu and Kamilov}]{liu_image_2019}
\bibinfo{author}{Liu, J.}, \bibinfo{author}{Sun, Y.}, \bibinfo{author}{Xu, X.},
  \bibinfo{author}{Kamilov, U.S.}, \bibinfo{year}{2019}.
\newblock \bibinfo{title}{Image restoration using total variation regularized
  deep image prior}, in: \bibinfo{booktitle}{IEEE International Conference on
  Acoustics, Speech and Signal Processing (ICASSP)}, pp.
  \bibinfo{pages}{7715--7719}.
\bibitem[{Mataev et~al.(2019)Mataev, Milanfar and Elad}]{mataev_deepred_2019}
\bibinfo{author}{Mataev, G.}, \bibinfo{author}{Milanfar, P.},
  \bibinfo{author}{Elad, M.}, \bibinfo{year}{2019}.
\newblock \bibinfo{title}{Deepred: Deep image prior powered by red}, in:
  \bibinfo{booktitle}{Proceedings of the IEEE/CVF International Conference on
  Computer Vision Workshops}, pp. \bibinfo{pages}{0--0}.
\bibitem[{Maulen-Soto et~al.(2024)Maulen-Soto, Fadili and Ochs}]{Maulen24}
\bibinfo{author}{Maulen-Soto, R.}, \bibinfo{author}{Fadili, J.},
  \bibinfo{author}{Ochs, P.}, \bibinfo{year}{2024}.
\newblock \bibinfo{title}{Inertial methods with viscous and {H}essian driven
  damping for non-convex optimization}.
\newblock \bibinfo{howpublished}{arXiv:2407.12518}.
\bibitem[{Monga et~al.(2021)Monga, Li and Eldar}]{monga_algorithm_2021}
\bibinfo{author}{Monga, V.}, \bibinfo{author}{Li, Y.}, \bibinfo{author}{Eldar,
  Y.C.}, \bibinfo{year}{2021}.
\newblock \bibinfo{title}{Algorithm unrolling: Interpretable, efficient deep
  learning for signal and image processing}.
\newblock \bibinfo{journal}{IEEE Signal Processing Magazine}
  \bibinfo{volume}{38}, \bibinfo{pages}{18--44}.
\bibitem[{Nesterov(2013)}]{nesterov2013introductory}
\bibinfo{author}{Nesterov, Y.}, \bibinfo{year}{2013}.
\newblock \bibinfo{title}{Introductory lectures on convex optimization: A basic
  course}. volume~\bibinfo{volume}{87}.
\newblock \bibinfo{publisher}{Springer Science \& Business Media}.
\bibitem[{Ongie et~al.(2020)Ongie, Jalal, Metzler, Baraniuk, Dimakis and
  Willett}]{ongie_deep_2020}
\bibinfo{author}{Ongie, G.}, \bibinfo{author}{Jalal, A.},
  \bibinfo{author}{Metzler, C.A.}, \bibinfo{author}{Baraniuk, R.G.},
  \bibinfo{author}{Dimakis, A.G.}, \bibinfo{author}{Willett, R.},
  \bibinfo{year}{2020}.
\newblock \bibinfo{title}{Deep {Learning} {Techniques} for {Inverse} {Problems}
  in {Imaging}}.
\newblock \bibinfo{journal}{IEEE Journal on Selected Areas in Information
  Theory} , \bibinfo{pages}{39--56}.
\bibitem[{Oymak and Soltanolkotabi(2019)}]{oymak_overparameterized_2019}
\bibinfo{author}{Oymak, S.}, \bibinfo{author}{Soltanolkotabi, M.},
  \bibinfo{year}{2019}.
\newblock \bibinfo{title}{Overparameterized nonlinear learning: Gradient
  descent takes the shortest path?}, in: \bibinfo{booktitle}{International
  Conference on Machine Learning}, pp. \bibinfo{pages}{4951--4960}.
\bibitem[{Oymak and Soltanolkotabi(2020)}]{oymak_toward_2020}
\bibinfo{author}{Oymak, S.}, \bibinfo{author}{Soltanolkotabi, M.},
  \bibinfo{year}{2020}.
\newblock \bibinfo{title}{Toward moderate overparameterization: Global
  convergence guarantees for training shallow neural networks}.
\newblock \bibinfo{journal}{IEEE Journal on Selected Areas in Information
  Theory} \bibinfo{volume}{1}, \bibinfo{pages}{84--105}.
\bibitem[{Pineda and Petersen(2023)}]{Petersen23}
\bibinfo{author}{Pineda, A.F.L.}, \bibinfo{author}{Petersen, P.C.},
  \bibinfo{year}{2023}.
\newblock \bibinfo{title}{Deep neural networks can stably solve
  high-dimensional, noisy, non-linear inverse problems}.
\newblock \bibinfo{journal}{Analysis and Applications} \bibinfo{volume}{21},
  \bibinfo{pages}{49--91}.
\bibitem[{Polyak(1964)}]{polyak_methods_1964}
\bibinfo{author}{Polyak, B.T.}, \bibinfo{year}{1964}.
\newblock \bibinfo{title}{Some methods of speeding up the convergence of
  iteration methods}.
\newblock \bibinfo{journal}{USSR Computational Mathematics and Mathematical
  Physics} \bibinfo{volume}{4}, \bibinfo{pages}{1--17}.
\bibitem[{Prost et~al.(2021)Prost, Houdard, Almansa and
  Papadakis}]{prost_learning_2021}
\bibinfo{author}{Prost, J.}, \bibinfo{author}{Houdard, A.},
  \bibinfo{author}{Almansa, A.}, \bibinfo{author}{Papadakis, N.},
  \bibinfo{year}{2021}.
\newblock \bibinfo{title}{Learning local regularization for variational image
  restoration}, in: \bibinfo{booktitle}{International Conference on Scale Space
  and Variational Methods in Computer Vision}, pp. \bibinfo{pages}{358--370}.
\bibitem[{Rahimi and Recht(2008)}]{rahimi_weighted_2008}
\bibinfo{author}{Rahimi, A.}, \bibinfo{author}{Recht, B.},
  \bibinfo{year}{2008}.
\newblock \bibinfo{title}{Weighted sums of random kitchen sinks: Replacing
  minimization with randomization in learning}.
\newblock \bibinfo{journal}{Advances in neural information processing systems}
  \bibinfo{volume}{21}.
\bibitem[{Shi et~al.(2022)Shi, Mettes, Maji and Snoek}]{shi_measuring_2022}
\bibinfo{author}{Shi, Z.}, \bibinfo{author}{Mettes, P.}, \bibinfo{author}{Maji,
  S.}, \bibinfo{author}{Snoek, C.G.}, \bibinfo{year}{2022}.
\newblock \bibinfo{title}{On measuring and controlling the spectral bias of the
  deep image prior}.
\newblock \bibinfo{journal}{International Journal of Computer Vision}
  \bibinfo{volume}{130}, \bibinfo{pages}{885--908}.
\bibitem[{Tirer et~al.(2024)Tirer, Giryes, Chun and Eldar}]{Tirer24}
\bibinfo{author}{Tirer, T.}, \bibinfo{author}{Giryes, R.},
  \bibinfo{author}{Chun, S.Y.}, \bibinfo{author}{Eldar, Y.C.},
  \bibinfo{year}{2024}.
\newblock \bibinfo{title}{Deep internal learning: Deep learning from a single
  input}.
\newblock \bibinfo{journal}{IEEE Signal Processing Magazine}
  \bibinfo{volume}{41}, \bibinfo{pages}{40--57}.
\bibitem[{Ulyanov et~al.(2018)Ulyanov, Vedaldi and
  Lempitsky}]{ulyanov_deep_2020}
\bibinfo{author}{Ulyanov, D.}, \bibinfo{author}{Vedaldi, A.},
  \bibinfo{author}{Lempitsky, V.}, \bibinfo{year}{2018}.
\newblock \bibinfo{title}{Deep image prior}, in:
  \bibinfo{booktitle}{Proceedings of the IEEE conference on computer vision and
  pattern recognition}, pp. \bibinfo{pages}{9446--9454}.
\bibitem[{Vershynin(2012)}]{vershynin_introduction_2012}
\bibinfo{author}{Vershynin, R.}, \bibinfo{year}{2012}.
\newblock \bibinfo{title}{Introduction to the non-asymptotic analysis of random
  matrices}, in: \bibinfo{booktitle}{Compressed {Sensing}: {Theory} and
  {Applications}}. \bibinfo{publisher}{Cambridge University Press},
  \bibinfo{address}{Cambridge}, pp. \bibinfo{pages}{210--268}.
\bibitem[{Zukerman et~al.(2021)Zukerman, Tirer and
  Giryes}]{zukerman_bp-dip_2021}
\bibinfo{author}{Zukerman, J.}, \bibinfo{author}{Tirer, T.},
  \bibinfo{author}{Giryes, R.}, \bibinfo{year}{2021}.
\newblock \bibinfo{title}{Bp-dip: A backprojection based deep image prior}, in:
  \bibinfo{booktitle}{28th European Signal Processing Conference},
  \bibinfo{organization}{IEEE}. pp. \bibinfo{pages}{675--679}.

\end{thebibliography}

\end{document}